\documentclass[12pt]{article}
\usepackage[round,colon,authoryear]{natbib}
\usepackage{bibentry}
\RequirePackage[OT1]{fontenc}
\RequirePackage{amsthm,amsmath,amsfonts,amssymb,{mathtools}}
\RequirePackage[colorlinks]{hyperref}
\usepackage{bbm}
\hypersetup{
 colorlinks=false,
 citecolor=Violet,
 linkcolor=Red,
 urlcolor=Blue}
\RequirePackage{epsfig, graphicx, color}
\RequirePackage{latexsym}
\RequirePackage{psfrag}

\usepackage{fullpage}
\usepackage{epsf,epsfig,subfigure,caption}
\usepackage{graphics,graphicx}
\usepackage{epstopdf}
\usepackage{algorithm,algorithmicx,algpseudocode}

\newtheorem{theorem}{Theorem}
\newtheorem{lemma}{Lemma}

\newtheorem{corollary}{Corollary}

\renewcommand{\hat}{\widehat}

\renewcommand{\hat}{\widehat}

\newcommand{\bfm}[1]{\ensuremath{\mathbf{#1}}}

   \def\bA{\bfm A}  
   \def\bB{\bfm B}  
   \def\bC{\bfm C}  
   \def\bD{\bfm D}  
\def\be{\bfm e}   \def\bE{\bfm E}  \def\EE{\mathbb{E}}
    
   \def\bG{\bfm G}  
   \def\bH{\bfm H}  
   \def\bI{\bfm I}  \def\II{\mathbb{I}}

   \def\bL{\bfm L}  
   \def\bM{\bfm M}  
   \def\bN{\bfm N}  
     
   \def\bP{\bfm P}  \def\PP{\mathbb{P}}
   \def\bQ{\bfm Q}  
\def\br{\bfm r}   \def\bR{\bfm R}  \def\RR{\mathbb{R}}
   \def\bS{\bfm S}  
   \def\bT{\bfm T}  
\def\bu{\bfm u}   \def\bU{\bfm U}  
\def\bv{\bfm v}   \def\bV{\bfm V}  
\def\bw{\bfm w}   \def\bW{\bfm W}  
\def\bx{\bfm x}   \def\bX{\bfm X}  
\def\by{\bfm y}   \def\bY{\bfm Y}  
\def\bz{\bfm z}   \def\bZ{\bfm Z}  

\def\calA{{\cal  A}}

\def\calE{{\cal  E}} 
 
\def\calG{{\cal  G}} 
\def\calH{{\cal  H}} 
\def\calI{{\cal  I}} 
 
\def\calK{{\cal  K}} 
 
\def\calM{{\cal  M}} 
\def\calN{{\cal  N}} 
\def\calO{{\cal  O}} 
\def\calP{{\cal  P}}

\def\calT{{\cal  T}} 
\def\calU{{\cal  U}} 
\def\calV{{\cal  V}} 
\def\calW{{\cal  W}} 
\def\calX{{\cal  X}}

\DeclareMathOperator{\rank}{rank}

\DeclareMathOperator{\tr}{tr}

\def\newpage{\vfill\eject}

\newdimen\biblioindent    \biblioindent=30pt

 at 8truept

\newcommand{\beq}{\begin{equation}}
  \newcommand{\eeq}{\end{equation}}
\newcommand{\beqn}{\begin{eqnarray}}
  \newcommand{\eeqn}{\end{eqnarray}}
\newcommand{\beqnn}{\begin{eqnarray*}}
  \newcommand{\eeqnn}{\end{eqnarray*}}

\newcounter{CondCounter}
\usepackage{mathtools}

\DeclarePairedDelimiter\floor{\lfloor}{\rfloor}

\begin{document}
\title{On Polynomial Time Methods for Exact Low Rank Tensor Completion$^\ast$}
\author{Dong Xia and Ming Yuan\\
Morgridge Institute for Research and University of Wisconsin-Madison}
\date{(\today)}

\maketitle

\footnotetext[1]{
This research was supported by NSF FRG Grant DMS-1265202, and NIH Grant 1U54AI117924-01.}

\begin{abstract}
In this paper, we investigate the sample size requirement for exact recovery of a high order tensor of low rank from a subset of its entries. We show that a gradient descent algorithm with initial value obtained from a spectral method can, in particular, reconstruct a ${d\times d\times d}$ tensor of multilinear ranks $(r,r,r)$ with high probability from as few as $O(r^{7/2}d^{3/2}\log^{7/2}d+r^7d\log^6d)$ entries. In the case when the ranks $r=O(1)$, our sample size requirement matches those for nuclear norm minimization \citep{yuan2015tensor}, or alternating least squares assuming orthogonal decomposability \citep{jain2014provable}. Unlike these earlier approaches, however, our method is efficient to compute, easy to implement, and does not impose extra structures on the tensor. Numerical results are presented to further demonstrate the merits of the proposed approach.
\end{abstract}

\newpage

\section{Introduction}
Let $\bT\in \RR^{d_1\times\cdots\times d_k}$ be a $k$th order tensor. The goal of tensor completion is to recover $\bT$ based on a subset of its entries $\{T(\omega): \omega\in \Omega\}$ for some $\Omega\subset [d_1]\times\cdots\times [d_k]$ where $[d]=\{1,2,\ldots, d\}$. The problem of tensor completion has attracted a lot of attention in recent years due to its wide range of applications. See, e.g. \cite{LiLi10,SidNion10, tomioka2010estimation, gandy2011tensor, CohenCollins12, liu2013tensor, anandkumar2014tensor, mu2014square, Semerci14,yuan2015tensor} and references therein. In particular, the second order (matrix) case has been extensively studied. See, e.g. \cite{candes2009exact, keshavan2009matrix, candes2010power, gross2011recovering, recht2011simpler} among many others. One of the main revelations from these studies is that, although the matrix completion problem is in general NP-hard, it is possible to develop tractable algorithms to achieve exact recovery with high probability. Naturally one asks if the same can be said for higher order tensors. This seemingly innocent task of generalizing from second order to higher order tensors turns out to be rather delicate.

The challenges in dealing with higher order tensors comes from both computational and theoretical fronts. On the one hand, many of the standard operations for matrices become prohibitively expensive to compute for higher order tensors. A notable example is the computation of tensor spectral norm. For second order tensors, or matrices, the spectral norm is merely its largest singular value and can be computed with little effort. Yet this is no longer the case for higher order tensors where computing the spectral norm is NP-hard in general \citep[see, e.g.,][]{hillarlim13}. On the other hand, many of the mathematical tools, either algebraic such as characterizing the subdifferential of the nuclear norm or probabilistic such as concentration inequalities, essential to the analysis of matrix completion are still under development for higher order tenors. There is a fast growing literature to address both issues and much progresses have been made in both fronts in the past several years.

When it comes to higher order tensor completion, an especially appealing idea is  to first unfold a tensor to a matrix and then treat it using techniques for matrix completion. Notable examples include \cite{tomioka2010estimation, gandy2011tensor, liu2013tensor, mu2014square} among others. As shown recently by \cite{yuan2015tensor}, these approaches, although easy to implement, may require an unnecessarily large amount of entries to be observed to ensure exact recovery. As an alternative, \cite{yuan2015tensor} established a sample size requirement for recovering a third order tensor via nuclear norm minimization and showed that a $d\times d\times d$ tensor with multilinear ranks $(r,r,r)$ can be recovered exactly with high probability with as few as $O((r^{1/2}d^{3/2}+r^2d)(\log d)^2)$ entries observed. Perhaps more surprisingly, \cite{yuan2016incoherent} later showed that the dependence on $d$ (e.g., the factor $d^{3/2}$) remains the same for higher order tensors and we can reconstruct a $k$th order cubic tensor with as few as $O((r^{(k-1)/2}d^{3/2}+r^{k-1}d)(\log d)^2)$ entries for any $k\ge 3$ when minimizing a more specialized nuclear norm devised to take into account the incoherence. These sample size requirement drastically improve those based on unfolding which typically require a sample size of the order $r^{\lfloor k/2\rfloor}d^{\lceil k/2\rceil}{\rm polylog}(d)$ \citep[see, e.g.,][]{mu2014square}. Although both nuclear norm minimization approaches are based on convex optimization, they are also NP hard to compute in general. Many approximate algorithms have also been proposed in recent years with little theoretical justification. See, e.g., \cite{kressner2014low,rauhut2015tensor,rauhut2016low}. It remains unknown if there exist polynomial time algorithms that can recover a low rank tensor exactly with similar sample size requirements. The goal of the present article is to fill in the gap between these two strands of research by developing a computationally efficient approach with tight sample size requirement for completing a third order tensor.

In particular, we show that there are polynomial time algorithms that can reconstruct a $d_1\times d_2\times d_3$ tensor with multilinear ranks $(r_1,r_2,r_3)$ from as few as
$$
O\left(r_1r_2r_3(rd_1d_2d_3)^{1/2}\log^{7/2}d+(r_1r_2r_3)^2rd \log^6d\right)
$$
entries where $r=\max\{r_1,r_2,r_3\}$ and $d=\max\{d_1,d_2,d_3\}$. This sample size requirement matches those for tensor nuclear norm minimization in terms of its dependence on the dimension $d_1, d_2$ and $d_3$ although it is inferior in terms of its dependence on the ranks $r_1, r_2$ and $r_3$. This makes our approach especially attractive in practice because we are primarily interested in high dimension (large $d$) and low rank (small $r$) instances. In particular, when $r=O(1)$, our algorithms can recover a tensor exactly based on $O(d^{3/2}\log^{7/2} d)$ observed entries, which is nearly identical to that based on nuclear norm minimization.

It is known that the problem of tensor completion can be cast as optimization over a direct product of Grassmannians \citep[see, e.g.,][]{kressner2014low}. The high level idea behind our development is similar to those used earlier by \cite{keshavan2009matrix} for matrix completion: if we can start with an initial value sufficiently close to the truth, then a small number of observed entries can ensure the convergence of typical optimization algorithms on Grassmannians such as gradient descent to the truth. Yet the implementation of this strategy is much more delicate and poses significant new challenges when moving from matrices to tensors.

At the core of our method is the initialization of the linear subspaces in which the fibers of a tensor reside. In the matrix case, a natural way to do so is by singular value decomposition, a tool that is no longer available for higher order tensors. An obvious solution is to unfold tensors into matrices and then applying the usual singular value decomposition based approach. This, however, requires an unnecessarily large sample size. To overcome this problem, we propose an alternative approach to estimating the singular spaces of the matrix unfoldings of a tensor. Our method is based on a carefully constructed estimate of the second moment of appropriate unfolding of a tensor, which can be viewed as a matrix version U-statistics. We show that the eigenspace of the proposed estimate concentrates around the true singular spaces of the matrix unfolding more sharply than the usual singular value decomposition based approaches, and therefore leads to consistent estimate with tighter sample size requirement.

The fact that there exist polynomial time algorithms to estimate a tensor consistently, not exactly, with $O(d^{3/2}{\rm polylog}(r,\log d))$ observed entries was first recognized by \cite{barak2016noisy}. Their approach is based on sum-of-square relaxations of tensor nuclear norm. Although polynomial time solvable in principle, their method requires solving a semidefinite program of size $d^3\times d^3$ and is not amenable to practical implementation. In contrast, our approach is essentially based on the spectral decomposition of a $d\times d$ matrix and can be computed fairly efficiently. Very recently, in independent work and under further restrictions on the tensor ranks, \cite{montanari2016spectral} showed that a spectral method different from ours can also achieve consistency with $O(d^{3/2}{\rm polylog}(r,\log d))$ observed entries. The rate of concentration for their estimate, however, is slower than ours and as a result, it is unclear if it provides a sufficiently accurate initial value for the exact recovery with the said sample size.

Once a good initial value is obtained, we consider reconstructing a tensor by optimizing on a direct product of Grassmannians locally. To this end, we consider a simple gradient descent algorithm adapted for our purposes. The main architect of our argument is similar to those taken by \cite{keshavan2009matrix} for matrix completion. We argue that the objective function, in a suitable neighbor around the truth and including the initial value, behaves like a parabola. As a result, the gradient descent algorithm necessarily converges locally to a stationary point. We then show that the true tensor is indeed the only stationary point in the neighborhood and therefore the algorithm recovers the truth. To prove these statements for higher order tensors however require a number of new probabilistic tools for tensors, and we do so by establishing several new concentration bounds, building upon those from \cite{yuan2015tensor,yuan2016incoherent}.

The rest of the paper is organized as follows. We first review necessary concepts and properties of tensors for our purpose in the next section. Section \ref{sec:main} describes our main result with the initialization and local optimization steps being treated in details in Sections \ref{sec:init} and \ref{sec:Flocal}. Numerical experiments presented in Section \ref{sec:sim} complement our theoretical development. We conclude with some discussions and remarks in Section \ref{sec:disc}. Proofs of the main results are presented in Section \ref{sec:proof}.

\section{Preliminaries}
\label{sec:prelim}

To describe our treatment of low rank tensor completion, we first review a few basic and necessary facts and properties of tensors. In what follows, we shall denote a tensor or matrix by a boldfaced upper-case letter, and its entries the same upper-case letter in normal font with appropriate indices. Similarly, a vector will be denoted by a boldfaced lower-case letter, and its entries by the same letter in normal font. For notational simplicity, we shall focus primarily on third order ($k=3$) tensors. Although our discussion can mostly be extended to higher order tensor straightforwardly. Subtle differences in treatment between third and higher order tensors will be discussed in Section \ref{sec:disc}.

The goal of tensor completion is to recover a tensor from partial observations of its entries. The problem is obviously underdetermined in general. To this end, we focus here on tensors that are of low multilinear ranks.

For a tensor $\bA\in \RR^{d_1\times d_2\times d_3}$, define the matrix $\calM_1(\bA)\in \RR^{d_1\times (d_2d_3)}$ by the entries
$$
\calM_1(\bA)(i_1,(i_2-1)d_3+i_3)=A(i_1,i_2,i_3),\qquad \forall i_1\in[d_1], i_2\in[d_2], i_3\in[d_3].
$$
In other words, the columns of $\calM_1(\bA)$ are the mode-1 fibers, $\{(A(i_1,i_2,i_3))_{i_1\in [d_1]}: i_2\in [d_2], i_3\in [d_3]\}$, of $\bA$. We can define $\calM_2$ and $\calM_3$ in the same fashion. It is clear that $\calM_j: \RR^{d_1\times d_2\times d_3}\to \RR^{d_j\times (d_1d_2d_3/d_j)}$ is a vector space isomorphism and often referred to as matricization or unfolding. The multilinear ranks of $\bA$ are given by
\begin{eqnarray*}
r_1(\bA)&=&\rank(\calM_1(\bA))={\rm dim}({\rm span}\{(A(i_1,i_2,i_3))_{i_1\in [d_1]}: i_2\in [d_2], i_3\in [d_3]\}),\\
r_2(\bA)&=&\rank(\calM_2(\bA))={\rm dim}({\rm span}\{(A(i_1,i_2,i_3))_{i_2\in [d_2]}: i_1\in [d_1], i_3\in [d_3]\}),\\
r_3(\bA)&=&\rank(\calM_3(\bA))={\rm dim}({\rm span}\{(A(i_1,i_2,i_3))_{i_3\in [d_3]}: i_1\in [d_1], i_2\in [d_2]\}).
\end{eqnarray*}
Note that, in general, $r_1(\bA)\neq r_2(\bA)\neq r_3(\bA)$.

Let $\bU$, $\bV$ and $\bW$ be the left singular vectors of $\calM_1(\bA)$, $\calM_2(\bA)$ and $\calM_3(\bA)$ respectively. It is not hard to see that there exists a so-called core tensor $\bC\in \RR^{r_1(\bA)\times r_2(\bA)\times r_3(\bA)}$ such that
\begin{equation}
\label{eq:tucker}
\bA=\sum_{j_1=1}^{r_1(\bA)}\sum_{j_2=1}^{r_2(\bA)}\sum_{j_3=1}^{r_3(\bA)} C(j_1,j_2,j_3) (\bu_{j_1}\otimes \bv_{j_2}\otimes \bw_{j_3}),
\end{equation}
where $\bu_{j}$, $\bv_j$ and $\bw_j$ are the $j$th column of $\bU$, $\bV$ and $\bW$ respectively, and
$$
\bx\otimes \by\otimes \bz:=(x_{i_1}y_{i_2}z_{i_3})_{i_1\in[d_1], i_2\in[d_2], i_3\in[d_3]},
$$
is a so-called rank-one tensor. Following the notation from \cite{silvalim08}, \eqref{eq:tucker} can also be more compactly represented as a trilinear multiplication:
$$
\bA=(\bU, \bV, \bW)\cdot \bC:=\bC\times_1\bU\times_2\bV\times_3\bW,
$$
where the marginal product $\times_1:\RR^{r_1\times r_2\times r_3}\times \RR^{d_1\times r_1}\to \RR^{d_1\times r_2\times r_3}$ is given by
$$
\bA\times_1\bB=\left(\sum_{j_1=1}^{r_1} A(j_1,j_2,j_3)B(i_1,j_1)\right)_{i_1\in [d_1], j_2\in [r_2], j_3\in [r_3]},
$$
and $\times_2$ and $\times_3$ are similarly defined.

The collection of all tensors of dimension $d_1\times d_2\times d_3$ whose multilinear ranks are at most $\br=(r_1,r_2,r_3)$ can be written as
$$
\calA(\br)=\left\{(\bX, \bY, \bZ)\cdot \bC: \bX\in \calV(d_1,r_1), \bY\in \calV(d_2,r_2), \bZ\in \calV(d_3,r_3), \bC\in \RR^{r_1\times r_2\times r_3}\right\},
$$
where $\calV(d,r)$ is the Stiefel manifold of orthonormal $r$-frames in $\RR^d$. In fact, any tensor $\bA\in \calA(\br)$ can be identified with a $r_1$ dimensional linear subspace in $\RR^{d_1}$, a $r_2$ dimensional linear subspace in $\RR^{d_2}$, a $r_3$ dimensional linear subspace in $\RR^{d_3}$ and a core tensor in $\RR^{r_1\times r_2\times r_3}$ so that $\calA(\br)$ is isomorphic to $\calG(d_1,r_1)\times\calG(d_2,r_2)\times\calG(d_3,r_3)\times \RR^{r_1\times r_2\times r_3}$ where $\calG(d,r)$ is the Grassmannian of $r$-dimensional linear subspaces in $\RR^d$.

Another common way of defining tensor ranks is through the so-called CP decomposition which expresses a tensor as the sum of the smallest possible number of rank-one tensors. The number of rank-one tensors in the CP decomposition of a tensor is commonly referred to as its CP rank. It is not hard to see that for a tensor of multilinear ranks $(r_1,r_2,r_3)$, its CP rank is necessarily between $\max\{r_1,r_2,r_3\}$ and $\min\{r_1r_2,r_1r_3,r_2r_3\}$. We shall focus here primarily on multilinear ranks because it allows for stable numerical computation, as well as refined theoretical analysis. But our results can be straightforwardly translated into CP ranks through the relationship between multilinear ranks and CP rank.

%where $\bP_\calU\in \RR^{d\times d}$ stands for the projection matrix of a $r$ dimensional linear subspace $\calU$ in $\RR^d$, $\calG(d,r)=\calV(d,r)/\calO(r)$ and $\calV(d,r)$ $\calV(d,r)$ represents the Stiefel manifold of orthonormal $k$-frames in $\RR^d$. 
%of represents the Grassmannian of $r$-dimensional linear subspaces in $\RR^d$.

In addition to being of low rank, another essential property that $\bT$ needs to satisfy so that we can possibly recover it from a uniformly sampled subset of its entries is the incoherence of linear subspaces spanned by its fibers \citep[see, e.g.,][]{candes2009exact}. More specifically, let $\calX$ be a $r$ dimensional linear subspace in $\RR^d$ and $\bP_\calX: \RR^d\to \RR^d$ be its projection matrix. We can define the coherence for $\calX$ as
$$
\mu(\calX)=\frac{d}{r}\max_{1\leq i\leq d}\left\|\bP_\calX \be_i\right\|^2,
$$
where $\be_i$ is the $i$th canonical basis of an Euclidean space, that is, it is a vector whose $i$th entry is one and all other entries are zero. Note that
$$
\mu(\calX)=\frac{\max_{1\leq i\leq d}\|\bP_\calX \be_i\|^2}{d^{-1}\sum_{i=1}^d\|\bP_\calX \be_i\|^2},
$$
for
$$
\sum_{i=1}^d\|\bP_\calX \be_i\|^2={\rm trace}(\bP_\calX)=r.
$$
Now for a tensor $\bA\in \RR^{d_1\times d_2\times d_3}$, denote by $\calU(\bA)$ the linear space spanned by its mode-1 fibers, $\calV(\bA)$ mode-2 fibers, and $\calW(\bA)$ mode-3 fibers. With slight abuse of notation, we define the coherence of $\bA$ as
$$
\mu(\bA)=\max\left\{\mu(\calU(\bA)), \mu(\calV(\bA)),\mu(\calW(\bA))\right\}.
$$
%We shall denote by
%$$
%\calA(\br, \mu_0)=\{\bA\in \calA(\br): \mu(\bA)\le \mu_0\}
%$$
%the collection of low rank tensors whose coherence is upper bounded by $\mu_0$.

In what follows, we shall also encounter various tensor norms. Recall that the vector-space inner product between two tensors $\bX,\bY\in\mathbb{R}^{d_1\times d_2\times d_3}$ is defined as
$$
\langle \bX, \bY\rangle=\sum_{\omega\in [d_1]\times[d_2]\times[d_3]} X(\omega)Y(\omega).
$$
The corresponding norm, referred to as Frobenius norm, for a tensor $\bA\in \RR^{d_1\times d_2\times d_3}$ is given by
$$\|\bA\|_{\rm F}:=\langle\bA,\bA\rangle^{1/2}.$$
We can also define the spectral norm of $\bA$ as
$$
\|\bA\|:=\sup_{\bu_j\in \mathbb{R}^{d_j}: \|\bu_1\|=\|\bu_2\|=\|\bu_3\|=1}\langle\bA, \bu_1\otimes \bu_2\otimes \bu_3\rangle,
$$
where, with slight abuse of notation, we write $\|\cdot\|$ both as the spectral norm for a tensor and as the usual $\ell_2$ norm for a vector for brevity. The nuclear nom is the dual of spectral norm:
$$
\|\bA\|_{\star}=\underset{\bX\in\mathbb{R}^{d_1\times d_2\times d_3}, \|\bX\|\leq 1}{\sup}\langle \bA, \bX\rangle.
$$
Another norm of interest is the max norm or the entrywise sup norm of $\bA$:
$$
\|\bA\|_{\max}:=\max_{\omega\in[d_1]\times [d_2]\times [d_3]}\left|A(\omega)\right|.
$$
The following relationships among these norms are immediate and are stated here for completeness. We shall make use of them without mentioning throughout the rest of our discussion.
%It is clear that for a tensor $\bA\in\mathbb{R}^{d_1\times d_2\times d_3}$
%$$
%\|\bA\|_{\max}\leq \|\bA\|\leq \|\bA\|_{\rm F}\le \sqrt{r_1(\bA)r_2(\bA)r_3(\bA)}\|\bA\|,
%$$
%where the last inequality follows from the fact that
%$$\|\bA\|_{\rm F}=\|\bC\|_{\rm F}\le \sqrt{r_1(\bA)r_2(\bA)r_3(\bA)}\|\bC\|_{\max}\le \sqrt{r_1(\bA)r_2(\bA)r_3(\bA)}\|\bA\|,$$
%and $\bC$ is given by \eqref{eq:tucker}. The following lemma will be needed.
\begin{lemma}\label{lemma:tensornorm}
For any $\bA\in\mathbb{R}^{d_1\times d_2\times d_3}$,
$$
\|\bA\|_{\max}\leq \|\bA\|\leq \|\bA\|_{\rm F}\le \sqrt{r_1(\bA)r_2(\bA)r_3(\bA)}\|\bA\|,
$$
and
$$
\|\bA\|_{\star}\leq \min\Big\{\sqrt{r_1(\bA)r_2(\bA)},\sqrt{r_1(\bA)r_3(\bA)},\sqrt{r_2(\bA)r_3(\bA)}\Big\}\|\bA\|_{\rm F}.
$$
\end{lemma}
\medskip

The proof of Lemma \ref{lemma:tensornorm} is included in the Appendix~\ref{sec:tensornorm} for completeness. We are now in position to describe our approach to tensor completion.
\section{Tensor Completion}
\label{sec:main}

Assume that $\bT$ has multilinear ranks $\br:=(r_1,r_2,r_3)$ and coherence at most $\mu_0$, we want to recover $\bT$ based on $\big(\omega_i ,T(\omega_i)\big)$ for $i=1,2,\ldots, n$ where $\omega_i$ are independently and uniformly drawn from $[d_1]\times[d_2]\times[d_3]$. This sampling scheme is often referred to the Bernoulli model, or sampling with replacement \citep[see, e.g.,][]{gross2011recovering, recht2011simpler}. Another commonly considered scheme is the so-called uniform sampling without replacement where we observe $T(\omega)$ for $\omega\in \Omega$ and $\Omega$ is a uniformly sampled subset of $[d_1]\times[d_2]\times[d_3]$ with size $|\Omega|=n$. It is known that both sampling schemes are closely related in that, given a uniformly sampled subset $\Omega$ of size $n$, one can always create a sample $\omega_i\in \Omega$, $i=1,\ldots,n$ so that $\omega_i$s follow the Bernoulli model. This connection ensures that any method that works for Bernoulli model necessarily works for uniform sampling without replacement as well. From a technical point of view, it has been demonstrated that working with the Bernoulli model leads to considerably simpler arguments for a number of matrix or tensor completion approaches.  See, e.g., \cite{gross2011recovering, recht2011simpler, yuan2015tensor}, among others. For these reasons, we shall focus on the Bernoulli model in the current work.

%\begin{algorithm}
% \caption{Create IID samples from $\Omega$}\label{samplealg}
%  \begin{algorithmic}[5]
%  \State Initiate $S_0=\emptyset$.
%   \For{$t=1,\ldots,n$}
%   \State Sample $\omega_t$ uniformly from $S_{t-1}$ with probability $|S_{t-1}|/d_1d_2d_3$ and uniformly from $\Omega\setminus S_{t-1}$ with probability $1-\frac{|S_{t-1}|}{d_1d_2d_3}$.
%   \State Update $S_t=S_{t-1}\cup \{\omega_t\}$.
%   \EndFor
%  \end{algorithmic}
%\end{algorithm}

A natural way to solve this problem is through the following optimization:
$$
\min_{\bA\in \calA(\br)}  {1\over 2}\left\|\calP_\Omega(\bA-\bT)\right\|_{\rm F}^2.
$$
where the linear operator $\calP_\Omega: \RR^{d_1\times d_2\times d_3}\to \RR^{d_1\times d_2\times d_3}$ is given by
$$
\calP_\Omega \bX=\sum_{i=1}^n \calP_{\omega_i}\bX,%\begin{cases}X(\omega), &\textrm{if } \omega\in\Omega\\0, &\textrm{otherwise}\end{cases}.
$$
and $\calP_{\omega}\bX$ is a $d_1\times d_2\times d_3$ tensor whose $\omega$ entry is $X(\omega)$ and other entries are zero. Equivalently, we can reconstruct $\bT=(\bU, \bV, \bW)\cdot \bG$ by $\hat{\bT}:=(\hat{\bU},\hat{\bV}, \hat{\bW})\cdot \hat{\bG}$ where the tuple $(\hat{\bU},\hat{\bV}, \hat{\bW},\hat{\bG})$ solves
\begin{equation}
\label{eq:opt}
\min_{\bX\in \calV(d_1,r_1), \bY\in \calV(d_2,r_2), \bZ\in \calV(d_3,r_3), \bC\in \RR^{r_1\times r_2\times r_3}}{1\over 2}\left\|\calP_\Omega((\bX, \bY, \bZ)\cdot \bC-\bT)\right\|_{\rm F}^2.
\end{equation}
Recall that $\bX\otimes\bY\otimes \bZ$ is a sixth order tensor of dimension $d_1\times d_2\times d_3\times r_1\times r_2\times r_3$. With slight abuse of notation, for any $\omega\in [d_1]\times [d_2]\times[d_3]$, denote by $(\bX\otimes\bY\otimes \bZ)(\omega)$ a third order tensor with the first three indices of $\bX\otimes\bY\otimes \bZ$ fixed at $\omega$. By the first order optimality condition, we get
$$
\sum_{i=1}^n\left\langle(\bX\otimes\bY\otimes\bZ)(\omega_i),\bC\right\rangle (\bX\otimes\bY\otimes\bZ)(\omega_i) =\sum_{i=1}^n T(\omega_i)(\bX\otimes\bY\otimes\bZ)(\omega_i),
$$
so that
\begin{eqnarray}\nonumber
{\rm vec}(\bC)=\left(\sum_{i=1}^n{\rm vec}((\bX\otimes\bY\otimes\bZ)(\omega_i)){\rm vec}((\bX\otimes\bY\otimes\bZ)(\omega_i))^\top\right)^{-1}\times\\
\label{eq:defbC}
\left(\sum_{i=1}^n T(\omega_i){\rm vec}((\bX\otimes\bY\otimes \bZ)(\omega_i))\right).
\end{eqnarray}
Here, we assumed implicitly that $n\ge r_1r_2r_3$. In general, there may be multiple minimizers to \eqref{eq:opt} and we can replace the inverse by the Moore-Penrose pseudoinverse to yield a solution. Plugging it back to \eqref{eq:opt} suggests that $(\hat{\bU},\hat{\bV}, \hat{\bW})$ is the solution to
$$
\max_{\bX\in \calV(d_1,r_1), \bY\in \calV(d_2,r_2), \bZ\in \calV(d_3,r_3)}F(\bX, \bY, \bZ),
$$
where
\begin{eqnarray*}
F(\bX, \bY, \bZ)=\left(\sum_{i=1}^n T(\omega_i){\rm vec}((\bX\otimes \bY\otimes \bZ)(\omega_i))\right)^\top\times\\
\times\left(\sum_{i=1}^n{\rm vec}((\bX\otimes \bY\otimes \bZ)(\omega_i)){\rm vec}((\bX\otimes \bY\otimes \bZ)(\omega_i))^\top\right)^{-1}\left(\sum_{i=1}^n T(\omega_i){\rm vec}((\bX\otimes \bY\otimes \bZ)(\omega_i))\right).
\end{eqnarray*}

Let $\tilde{\bX}=\bX\bQ_1$, $\tilde{\bY}=\bY\bQ_2$ and $\tilde{\bZ}=\bZ\bQ_3$, where $\bQ_j\in \calO(r_j)$ and $\calO(r)$ is the set of $r\times r$ orthonormal matrices. It is easy to verify that
$$
F(\bX, \bY, \bZ)=F(\tilde{\bX}, \tilde{\bY}, \tilde{\bZ})
$$
so that it suffices to optimize $F(\bX, \bY, \bZ)$ over
$$(\bX,\bY,\bZ)\in (\calV(d_1,r_1)/\calO(r_1))\times (\calV(d_2,r_2)/\calO(r_2))\times (\calV(d_3,r_3)/\calO(r_3)).$$
Recall that $\calV(d,r)/\calO(r)\cong\calG(d,r)$, the Grassmaniann of $r$ dimensional linear subspace in $\RR^d$. Optimizing $F$ can then be cast an optimization problem over a direct product of Grassmanian manifolds, a problem that has been well studied in the literature. See, e.g., \cite{absiletal08}. In particular, (quasi-)Newton \citep[see, e.g.,][]{eldensavas09,savaslim10}, gradient descent \citep[see, e.g.,][]{keshavan2009matrix}, and conjugate gradient \citep[see, e.g.,][]{kressner2014low} methods have all been proposed previously to solve optimization problems similar to the one we consider here.

There are two prerequisites for any of these methods to be successful. The highly non-convex nature of the optimization problem dictates that even if any of the aforementioned iterative algorithms converges, it could only converge to a local optimum. Therefore a good initial value is critical. This unfortunately is an especially challenging task for tensors. For example, if we consider random initial values, then an prohibitively large number, in fact exponential in $d$, of seeds would be required to ensure the existence of a good starting point. Alternatively, in the second order or matrix case, \cite{keshavan2009matrix} suggests a singular value decomposition based approach for initialization. The method, however, cannot be directly applied for higher order tensors as similar type of spectral decomposition becomes NP hard to compute \citep{hillarlim13}. To address this challenge, we propose here a new spectral method that is efficient to compute and at the same time is guaranteed to produce an initial value sufficiently close to the optimal value.

With the initial value coming from a neighborhood near the truth, any of the aforementioned methods could then be applied in principle. In order for them to converge to the truth, we need to make sure that the objective function $F$ behaves well in the neighborhood. In particular, we shall show that, when $n$ is sufficiently large, $F$ behaves like a parabola in a neighborhood around the truth, and therefore ensures the local convergence of algorithms such as gradient descent.

We shall address both aspects, initialization and local convergence, separately in the next two sections. In summary, we can obtain a sample size requirement for exact recovery of $\bT$ via polynomial time algorithms. As in the matrix case, the sample size requirement depends on notions of condition number of $\bT$. Recall that the condition number for a matrix $\bA$ is given by $\kappa(\bA)=\sigma_{\max}(\bA)/\sigma_{\min}(\bA)$ where $\sigma_{\max}$ and $\sigma_{\min}$ are the largest and smallest nonzero singular values of $\bA$ respectively. We can straightforwardly generalize the concept to a third order tensor $\bA$ as:
$$\kappa(\bA)={\max\left\{\sigma_{\max}(\calM_1(\bA)), \sigma_{\max}(\calM_2(\bA)), \sigma_{\max}(\calM_3(\bA))\right\}\over \min\left\{\sigma_{\min}(\calM_1(\bA)), \sigma_{\min}(\calM_2(\bA)), \sigma_{\min}(\calM_3(\bA))\right\}}.$$
%A straightforward way to generalize the concept to a tensor $\bA\in \RR^{d_1\times d_2\times d_3}$ is by taking the maximum among condition numbers of its matrix unfoldings:
%$$
%\kappa(\bA):=\max\{\kappa(\calM_1(\bA)), \kappa(\calM_2(\bA)), \kappa(\calM_3(\bA))\}.
%$$
Our main result can then be summarized as follows:% (whose proof is provided in Section~\ref{sec:Flocal}):

%To this end, we clarify several notations. Define the eigengaps
%$$
%\lambda_{k}:=\sigma_{r_k}\big(\mathcal{M}_k(\bT)\big)=\sigma_{r_k}\big(\mathcal{M}_k(\bG)\big),\quad k=1,2,3
%$$
%where $\sigma_r(\bA)$ denote the $r$-th singular value of matrix $\bA$. Then, define $\lambda_{\min}=(\lambda_1\wedge \lambda_2\wedge \lambda_3)$.
%For simplicity, denote $\Lambda_{\max}=\|\bT\|$ the operator norm of $\bT$. Moreover, recall that $\bT=(\bU,\bV,\bW)\cdot \bG$ and define
%$$
%\Lambda_{\min}:=\underset{\bu_j\in\mathbb{R}^{r_j},\|\bu_1\|=\|\bu_2\|=\|\bu_3|=1}{\min}\big|\langle \bG, \bu_1\otimes \bu_2\otimes \bu_3\rangle\big|,
%$$
%i.e., the smallest singular value of core tensor $\bG$. Let $d^\star=(d_1\vee d_2\vee d_3)$ and $r^\star=(r_1\vee r_2\vee r_3)$.

\begin{theorem}
\label{thm:main}
Assume that $\bT\in \RR^{d_1\times d_2\times d_3}$ is a rank-$(r_1,r_2,r_3)$ tensor whose coherence is bounded by $\mu(\bT)\leq \mu_0$ and condition number is bounded by $\kappa(\bT)\le \kappa_0$. Then there exists a polynomial time algorithm that recovers $\bT$ exactly based on $\{\big(\omega_i, T(\omega_i)\big): 1\le i\le n\}$, with probability at least $1-d^{-\alpha}$ if $\omega_i$s are independently and uniformly sampled from $[d_1]\times [d_2]\times [d_3]$ and
\begin{equation}
\label{eq:samplesize}
n\geq C\left\{\alpha^3\mu_0^3\kappa_0^4r_1r_2r_3(rd_1d_2d_3)^{1/2}\log^{7/2}d+\alpha^6\mu_0^6\kappa_0^8(r_1r_2r_3)^2rd\log^6d\right\},
\end{equation}
for a universal constant $C>0$, and an arbitrary constant $\alpha\geq1$, where $d=\max\{d_1, d_2, d_3\}$ and $r=\max\{r_1,r_2, r_3\}$.
\end{theorem}
\medskip

%It is of interest to compare the sample size requirement (\ref{eq:samplesize}) with those in the current literature.
%
%To this end, we proceed in three steps.
%\begin{enumerate}
%\item[$\bullet$] We first show that a spectral method can produce an initial value sufficiently close to the truth;
%\item[$\bullet$] Next we show that the objective function $F$ behaves like a parabola in a neighborhood around the truth, and therefore ensures local convergence of algorithms such as gradient descent or conjugate gradient;
%\item[$\bullet$] Finally we argue that the truth is the only stationary point in its neighborhood.
%\end{enumerate}
%
%The next two sections will be devoted to each of these three steps.

\section{Second Order Method for Estimating Singular Spaces}
\label{sec:init}

We now describe a spectral algorithm that produces good initial values for $\bU$ and $\bV$ and $\bW$ based on $\calP_\Omega\bT$. To fix ideas, we focus on estimating $\bU$. $\bV$ and $\bW$ can be treated in an identical fashion. Denote by
$$
\hat{\bT}={d_1d_2d_3\over n}\calP_\Omega \bT.
$$
It is clear that $\EE(\hat{\bT})=\bT$ so that $\calM_1(\hat{\bT})$ is an unbiased estimate of $\calM_1(\bT)$. Recall that $\bU$ is the left singular vectors of $\calM_1(\bT)$, it is therefore natural to consider estimating $\bU$ by the leading singular vectors of $\calM_1(\hat{\bT})$. The main limitation of this na\"ive approach is its inability to take advantage of the fact that $\calM_1(\hat{\bT})$ may be unbalanced in that $d_1\ll d_2d_3$, and the quality of an estimate of $\bU$ is driven largely by the greater dimension ($d_2d_3$) although we are only interested in estimating the singular space in a lower dimensional ($d_1$) space.
%
%The sample size requirement for this na\"ive approach however is determined by the maximum, $\max\{d_1,d_2d_3\}$, of the two dimensions of $\calM_1(\bT)$, which could be excessive when $d_2d_3\gg d_1$ because $\bU$ represents a low dimensional linear space in $\RR^{d_1}$.

%Recall that $\bU$, $\bV$ and $\bW$ are the left singular spaces of $\calM_1(\bT)$, $\calM_2(\bT)$ and $\calM_3(\bT)$ respectively. It is therefore natural to consider estimating them by the corresponding singular spaces of $\calM_1(\hat{\bT})$, $\calM_2(\hat{\bT})$ and $\calM_3(\hat{\bT})$,
%where
%The performance of such an approach, however, depends 
%
%turns out not to work well. The problem already manifested itself in the matrix case. See, e.g.,\cite{keshavan2009matrix} to suggest a trimming strategy that discard all rows with too many observed entries. It is however unclear how a similar strategy may be extended to higher order tensors.

To specifically address this issue, we consider here a different technique for estimating singular spaces from a noisy matrix, which is more powerful when the underlying matrix is unbalanced in that it is either very fat or very tall. More specifically, let $\bM\in \RR^{m_1\times m_2}$ be a rank $r$ matrix. Our goal is to estimate the left singular space of $\bM$ based on $n$ pairs of observations $\{(\omega_i,\bM(\omega_i)): 1\le i\le n\}$ where $\omega_i$s are independently and uniformly sampled from $[m_1]\times [m_2]$. Recall that $\bU$ is also the eigenspace of $\bM\bM^\top$ which is of dimension $m_1\times m_1$. Instead of estimating $\bM$, we shall consider instead estimating $\bM\bM^\top$. To this end, write $\bX_i=(m_1m_2)\calP_{\omega_i}\bM$, that is a $m_1\times m_2$ matrix whose $\omega_i$ entry is $(m_1m_2)\bM(\omega_i)$, and other entries are zero. It is clear that $\EE\left(\bX_i\right)=\bM$. We shall then consider estimating $\bN:=\bM\bM^\top$ by
\begin{equation}
\label{eq:ustat}
\hat{\bN}:={1\over n(n-1)}\sum_{i<j} (\bX_i\bX_j^\top+\bX_j\bX_i^\top)
\end{equation}
Our first result shows that $\hat{\bN}$ could be a very good estimate of $\bN$ even in situations when $n\ll m_2$.

\begin{theorem}
\label{th:spectral}
Let $\bM\in \RR^{m_1\times m_2}$ and $\bX_i=(m_1m_2)\calP_{\omega_i}\bM$ ($i=1,2,\ldots,n$), where $\omega_i$s are independently and uniformly sampled from $[m_1]\times [m_2]$. There exists an absolute constant $C>0$ such that 
for any $\alpha>1$, if
$$n\ge {8\over 3}{(\alpha+1)\log m\over \min\{m_1, m_2\}},\qquad m:=\max\{m_1,m_2\}\geq 2$$
then
\begin{eqnarray*}
\|\hat{\bN}-\bM\bM^\top\|\le \\
C\cdot \alpha^2\cdot{m_1^{3/2}m_2^{3/2}\log m\over n}\left[\left(1+\frac{m_1}{m_2}\right)^{1/2}+{m_1^{1/2}m_2^{1/2}\over n}+\left({n\over m_2\log m}\right)^{1/2}\right]\cdot \|\bM\|_{\max}^2,
\end{eqnarray*}
with probability at least $1-m^{-\alpha}$, where $\hat{\bN}$ is given by \eqref{eq:ustat}.
\end{theorem}
\medskip

In particular, if $\|\bM\|_{\max}=O((m_1m_2)^{-1/2})$, then $\|\hat{\bN}-\bM\bM^\top\|\to_p 0$ as soon as $n\gg \big((m_1m_2)^{1/2}+m_1\big)\log m$. This is to be contrast with estimating $\bM$. As shown by \cite{recht2011simpler},
$$
\hat{\bM}:={1\over n}\sum_{i=1}^n \bX_i
$$
is a consistent estimate of $\bM$ in terms of spectral norm if $n\gg m\log m$. The two sample size requirements differ when $m_1\ll m_2$ in which case $\hat{\bN}$ is still a consistent estimate of $\bM\bM^\top$ yet $\hat{\bM}$ is no longer a consistent estimate of $\bM$ if $(m_1m_2)^{1/2}\log m_2\ll n\ll m_2\log m_2$.

Equipped with Theorem \ref{th:spectral}, we can now address the initialization of $\bU$ (and similarly $\bV$ and $\bW$). Instead of estimating it by the singular vectors of $\calM_1(\hat{\bT})$, we shall do so based on an estimate of $\calM_1(\bT)\calM_1(\bT)$. With slight abuse of notation, write $\bX_i=(d_1d_2d_3)\calM_1(\calP_{\omega_i}\bT)$ and
$$
\hat{\bN}:={1\over n(n-1)}\sum_{i<j} (\bX_i\bX_j^\top+\bX_j\bX_i^\top).
$$
We shall then estimate $\bU$ by the leading $r$ left singular vectors of $\hat{\bN}$, hereafter denoted by $\hat{\bU}$.

As we are concerned with the linear spaces spanned by the column vector of $\bU$ and $\hat{\bU}$ respectively, we can measure the estimation error by the projection distance defined over Grassmannian:
$$
d_{\rm p}(\bU,\hat{\bU}):=\frac{1}{\sqrt{2}}\|\bU\bU^{\top}-\hat{\bU}\hat{\bU}^\top\|_{\rm F}.
$$
The following result is an immediate consequence of Theorem \ref{th:spectral} and Davis-Kahn Theorem, and its proof is deferred to the Appendix.

\begin{corollary}
\label{co:init}
Assume that $\bT\in \RR^{d_1\times d_2\times d_3}$ is a rank-$(r_1,r_2,r_3)$ tensor whose coherence is bounded by $\mu(\bT)\leq \mu_0$ and condition number is bounded by $\kappa(\bT)\le \kappa_0$. Let $\bU$ be the left singular vectors of $\calM_1(\bT)$ and $\hat{\bU}$ be defined as above, then there exist absolute constants $C_1,C_2>0$ such that for any $\alpha>1$, if
$$n\ge C_1\left(\alpha(d_1d_2d_3)^{1/2}+d_1\log d\right),$$
then
\begin{eqnarray*}
d_{\rm p}(\bU,\hat{\bU})\le C_2\alpha^2\mu_0^3\kappa_0^2r_1^{3/2}r_2r_3\left({(d_1d_2d_3)^{1/2}\log d\over n}+\sqrt{d_1\log d\over n}\right),\end{eqnarray*}
with probability at least $1-d^{-\alpha}$.
\end{corollary}

In the light of Corollary \ref{co:init}, $\hat{\bU}$ (and similarly $\hat{\bV}$ and $\hat{\bW}$) is a consistent estimate of $\bU$ whenever
$$
n\gg \left[r_1^{3/2}r_2r_3(d_1d_2d_3)^{1/2}+r_1^3r_2^2r_3^2d\right]\log d.
$$

\section{Exact Recovery by Optimizing Locally}
\label{sec:Flocal}
%\section{Local geometry of $F(\bX,\bY,\bZ)$ over Grassmannians}\label{sec:Flocal}

Now that a good initial value sufficiently close to $(\bU, \bV,\bW)$ is identified, we can then proceed to optimize
$$
F(\bX,\bY,\bZ)=\min_{\bC\in \RR^{r_1\times r_2\times r_3}}{1\over 2}\|\calP_\Omega((\bX,\bY,\bZ)\cdot \bC-\bT)\|_{\rm F}^2
$$
locally. To this end, we argue that $F$ indeed is well-behaved in a neighborhood around $(\bU, \bV,\bW)$ so that such a local optimization is amenable to computation. For brevity, write
$$
\mathcal{J}(d_1,d_2,d_3,r_1,r_2,r_3):=\mathcal{G}(d_1,r_1)\times\mathcal{G}(d_2,r_2)\times\mathcal{G}(d_3,r_3).
$$
We can also generalize the projection distance $d_{\rm p}$ on Grassmaniann to $\mathcal{J}(d_1,d_2,d_3,r_1,r_2,r_3)$ as follows:
$$
d_{\rm p}\left((\bU,\bV,\bW), (\bX,\bY,\bZ)\right)=d_{\rm p}(\bU,\bX)+d_{\rm p}(\bV,\bY)+d_{\rm p}(\bW,\bZ).
$$
We shall focus, in particular, on a neighborhood around $(\bU,\bV,\bW)$ that are incoherent:
\begin{eqnarray*}
\calN(\delta,\mu)=\biggl\{(\bX,\bY,\bZ)\in \mathcal{J}(d_1,d_2,d_3,r_1,r_2,r_3): d_{\rm p}\left((\bU,\bV,\bW), (\bX,\bY,\bZ)\right)\le \delta,\\
{\rm and}\ \max\left\{\mu(\bX), \mu(\bY), \mu(\bZ)\right\}\le \mu\biggr\}
\end{eqnarray*}
For a third order tensor $\bA$, denote by
$$\Lambda_{\max}(\bA)=\max\left\{\sigma_{\max}(\calM_1(\bA)), \sigma_{\max}(\calM_2(\bA)), \sigma_{\max}(\calM_3(\bA))\right\},$$
and
$$\Lambda_{\min}(\bA)=\min\left\{\sigma_{\min}(\calM_1(\bA)), \sigma_{\min}(\calM_2(\bA)), \sigma_{\min}(\calM_3(\bA))\right\}.$$
%Our first result of the section shows that the objective function $F$ behaves like a parabola in $\calN(\delta,\mu)$ for sufficiently small $\delta$ and $\mu$. 
\begin{theorem}\label{thm:localF}
Let $\bT\in \RR^{d_1\times d_2\times d_3}$ be a rank-$(r_1,r_2,r_3)$ tensor such that
$$\mu(\bT)\leq \mu_0, \qquad \Lambda_{\min}(\bT)\ge \underline{\Lambda},\qquad \Lambda_{\max}(\bT)\le \overline{\Lambda},\qquad {\rm and}\qquad \kappa(\bT)\le \kappa_0.$$
There exist absolute constants $C_1,C_2,C_3, C_4,C_5>0$ such that for any $\alpha>1$ and $(\bX,\bY,\bZ)\in\mathcal{N}\big(C_1(\alpha\kappa_0\log d)^{-1},4\mu_0)$,
\begin{eqnarray*}
C_2\left(\|\bG-\bC\|_{\rm F}^2+\underline{\Lambda}^2 d^2_{\rm p}\left((\bU,\bV,\bW), (\bX,\bY,\bZ)\right)\right)\leq {d_1d_2d_3\over n}F(\bX,\bY,\bZ)\\
\leq C_3\alpha\overline{\Lambda}^2d^2_{\rm p}\left((\bU,\bV,\bW), (\bX,\bY,\bZ)\right)\log d,
\end{eqnarray*}
and
$$
{d_1d_2d_3\over n}\left\|{\rm grad}\ F(\bX,\bY,\bZ)\right\|_{\rm F}\geq C_4\left(\underline{\Lambda}^2 d_{\rm p}\left((\bU,\bV,\bW), (\bX,\bY,\bZ)\right)\right),
$$
with probability at least $1-3d^{-\alpha}$, provided that
$$
n\ge C_5\left\{\alpha^3\mu_0^{3/2}\kappa_0^4r(r_1r_2r_3d_1d_2d_3)^{1/2}\log^{7/2} d+\alpha^6\mu_0^3\kappa_0^8
r_1r_2r_3r^2d\log^6d\right\}
$$
where $\bC\in\mathbb{R}^{r_1\times r_2\times r_3}$ is given by \eqref{eq:defbC}.
\end{theorem}

Theorem \ref{thm:localF} shows that the objective function $F$ behaves like a parabola in $\calN(\delta,4\mu_0)$ for sufficiently small $\delta$, and furthermore, $(\bU,\bV,\bW)$ is the unique stationary point in $\calN(\delta,4\mu_0)$. This implies that a gradient descent type of algorithm may be employed to optimize $F$ within $\calN(\delta,4\mu_0)$. In particular, to fix ideas, we shall focus here on a simple gradient descent type of algorithms similar to the popular choice of matrix completion algorithm proposed by \cite{keshavan2009matrix}. As suggested by \cite{keshavan2009matrix}, to guarantee that the coherence condition is satisfied, a penalty function is imposed so that the objective function becomes:
$$
\tilde{F}(\bX,\bY,\bZ):=F(\bX,\bY,\bZ)+G(\bX,\bY,\bZ)
$$
where
$$
G(\bX,\bY,\bZ):=\rho \sum_{j_1=1}^{d_1}G_0\Big(\frac{d_1\|\bx_{j_1}\|^2}{3\mu_0r_1}\Big)+\rho \sum_{j_2=1}^{d_2}G_0\Big(\frac{d_2\|\by_{j_2}\|^2}{3\mu_0r_2}\Big)+\rho \sum_{j_3=1}^{d_3}G_0\Big(\frac{d_3\|\bz_{j_3}\|^2}{3\mu_0r_3}\Big)
$$
and
$$
G_0(z)=
\begin{cases}
0,&{\rm if }\ z\leq 1\\
e^{(z-1)^2}-1,&{\rm if }\ z\geq 1.
\end{cases}
$$
It turns out that, with a sufficiently large $\rho>0$, we can ensure low coherence at all iterations in a gradient descent algorithm. More specifically, let $\bB\in\mathbb{R}^{d\times r}$ be an element of the tangent space at $\bA\in\mathcal{G}(d,r)$ and $\bB=\bL{\bf \Theta}\bR^\top$ be its singular value decomposition. The geodesic starting from $\bA$ in the direction $\bB$ is defined as $\calH(\bA,\bB,t)=\bA\bR\cos({\bf \Theta}t)\bR^\top+\bL\sin({\bf \Theta}t)\bR^\top$ for $t\geq 0$. Interested readers are referred to \cite{edelman1998geometry} for further details on the differential geometry of Grassmannians. The gradient descent algorithm on the direct product of Grassmannians is given below:

\begin{algorithm}
 \caption{Gradient descent algorithm on Grassmannians (GoG)}\label{alg:GoGalg}
  \begin{algorithmic}[2]
%  \State Let $\bB\in\mathbb{R}^{d_1\times r_1}$ belongs to the tangent space at $\bA\in\mathcal{G}(d_1,r_1)$ and $\bB=\bL{\bf \Theta}\bR^\top$ be its singular value decomposition. The geodesic starting from $\bA$ in the directioin $\bB$ is defined as $\calH(\bA,\bB,t)=\bA\bR\cos({\bf \Theta}t)\bR^\top+\bL\sin({\bf \Theta}t)\bR^\top$ for $t\geq 0$.
  \State Set up values of $\max\_{\rm Iteration}$, tolerance $\varepsilon_{{\rm tol}}>0$, paramter $\gamma=\frac{\delta}{4}$ and step counter $k=0$.  
   \State Initiate $(\bX^{(0)}, \bY^{(0)}, \bZ^{(0)})=(\hat\bU,\hat\bV,\hat\bW)\in\mathcal{J}(d_1,d_2,d_3,r_1,r_2,r_3)$
   \While{$k<\max\_{\rm Iteration}$}
   \State Compute the negative gradient $(\bD_\bX^{(k)},\bD_\bY^{(k)},\bD_\bZ^{(k)})=-{\rm grad}\  \tilde{F}(\bX^{(k)}, \bY^{(k)}, \bZ^{(k)})$
   \State For $t\geq 0$, denote the geodesics
   \begin{align*}
    &\bX^{(k)}(t)=\calH(\bX^{(k)}, \bD_\bX^{(k)}, t)\\
    &\bY^{(k)}(t)=\calH(\bY^{(k)}, \bD_\bY^{(k)},t)\\
    &\bZ^{(k)}(t)=\calH(\bZ^{(k)}, \bD_\bZ^{(k)},t)
   \end{align*}
   \State Minimize $t\mapsto \tilde{F}(\bX^{(k)}(t), \bY^{(k)}(t), \bZ^{(k)}(t))$ for $t\geq 0$, subject to 
   $$
   d_{\rm p}\big(({\bX}^{(k)}(t),{\bY}^{(k)}(t),{\bZ}^{(k)}(t)), ({\bX}^{(0)},{\bY}^{(0)},{\bZ}^{(0)} ))\leq \gamma.
   $$
   \State Set $\bX^{(k+1)}=\bX^{(k)}(t_k), \bY^{(k+1)}=\bY^{(k)}(t_k)$ and $\bZ^{(k+1)}=\bZ^{(k)}(t_k)$ where $t_k$ is the minimal solution.
   \State Set $k=k+1$.
   \EndWhile
  \State Return $({\bX}^{(k+1)}, \bY^{(k+1)},\bZ^{(k+1)})$.
  \end{algorithmic}
\end{algorithm}

Our next result shows that this algorithm indeed converges to $(\bU,\bV, \bW)$ when an appropriate initial value is provided.

\begin{theorem}\label{th:GoGconv}
Let $\bT\in \RR^{d_1\times d_2\times d_3}$ be a rank-$(r_1,r_2,r_3)$ tensor such that
$$\mu(\bT)\leq \mu_0, \qquad \Lambda_{\max}(\bT)\le \overline{\Lambda},\qquad {\rm and}\qquad \kappa(\bT)\le \kappa_0.$$
Then there exist absolute constants $C_1,C_2, C_3>0$ such that for any $\alpha>1$, if
$$\rho\ge C_1\alpha n(d_1d_2d_3)^{-1}\overline{\Lambda}^2\log d,$$ 
$$(\bX^{(0)},\bY^{(0)},\bZ^{(0)})\in\mathcal{N}\big(C_2(\alpha\kappa_0^2\log d)^{-1},3\mu_0\big),$$
and
$$
n\ge C_3\left\{\alpha^3\mu_0^{3/2}\kappa_0^4r(r_1r_2r_3d_1d_2d_3)^{1/2}\log^{7/2} d+\alpha^6\mu_0^3\kappa_0^8
r_1r_2r_3r^2d\log^6d\right\},
$$
then Algorithm~\ref{alg:GoGalg} initiated with $(\bX^{(0)},\bY^{(0)},\bZ^{(0)})$ converges to $(\bU,\bV,\bW)$ with probability at least $1-d^{-\alpha}$.
\end{theorem}

%To show that Algorithm~\ref{alg:GoGalg} converges to $(\bU,\bV,\bW)$, we shall argue that the iterates are always incoherent and in a neighborhood around $(\bU,\bV,\bW)$. 

%The proof of Lemma~\ref{lem:firstorderlem} and Lemma~\ref{lem:secondorder} are provided in Section~\ref{sec:proof}.

\section{Numerical Experiments}
\label{sec:sim}

To complement our theoretical developments, we also conducted several sets of numerical experiments to investigate the performance of the proposed approach. In particular, we focus on recovering a cubic tensor $\bT\in\mathbb{R}^{d\times d\times d}$ with multilinear ranks $r_1=r_2=r_3=r$ from $n$ randomly sampled entries. To fix ideas, we focus on completing orthogonal decomposable tensors in this section, i.e., the core tensor $\bG\in\mathbb{R}^{r\times r\times r}$ is diagonal. Note that even though our theoretical analysis requires $n\gtrsim
r^{7/2}d^{3/2}$, our simulation results seem to suggest that our approach can be successful for as few as $O(\sqrt{r}d^{3/2})$ observed entries. To this end, we shall consider sample size $n=\alpha \sqrt{r}d^{3/2}$ for some $\alpha>0$.

More specifically, we consider $\bT=d\sum_{k=1}^r \bu_k\otimes \bv_k\otimes \bw_k\in\mathbb{R}^{d\times d\times d}$ with $d=50, 100$ and $r=2,3,4,5$. The orthonormal vectors $\{\bu_k, k=1,\ldots,r\},\{\bv_k, k=1,\ldots,r\},\{\bw_k, k=1,\ldots,r\}$ are obtained from the eigenspace of randomly generated standard Gaussian matrices which guarantee the incoherence conditions based on the delocalization property of eigenvectors of Gaussian random matrices. For each choice of $r$ and $\alpha=\frac{n}{\sqrt{r}d^{3/2}}$, the gradient descent algorithm from Section \ref{sec:Flocal} with initialization described in Section \ref{sec:init} is applied in $50$ simulation runs. We claim that the underlying tensor is successfully recovered if the returned tensor $\hat\bT$ satisfies that $\|\hat{\bT}-\bT\|_{\rm F}/\|\bT\|_{\rm F}\leq 10^{-7}$. The reconstruction rates are given in Figure~\ref{fig:fig1} and ~\ref{fig:fig2}. It suggests that approximately when $n\geq7\sqrt{r}d^{3/2}$, the algorithm reconstructed the true tensor with near certainty.

\begin{figure}
\centering
\includegraphics[width=15cm]{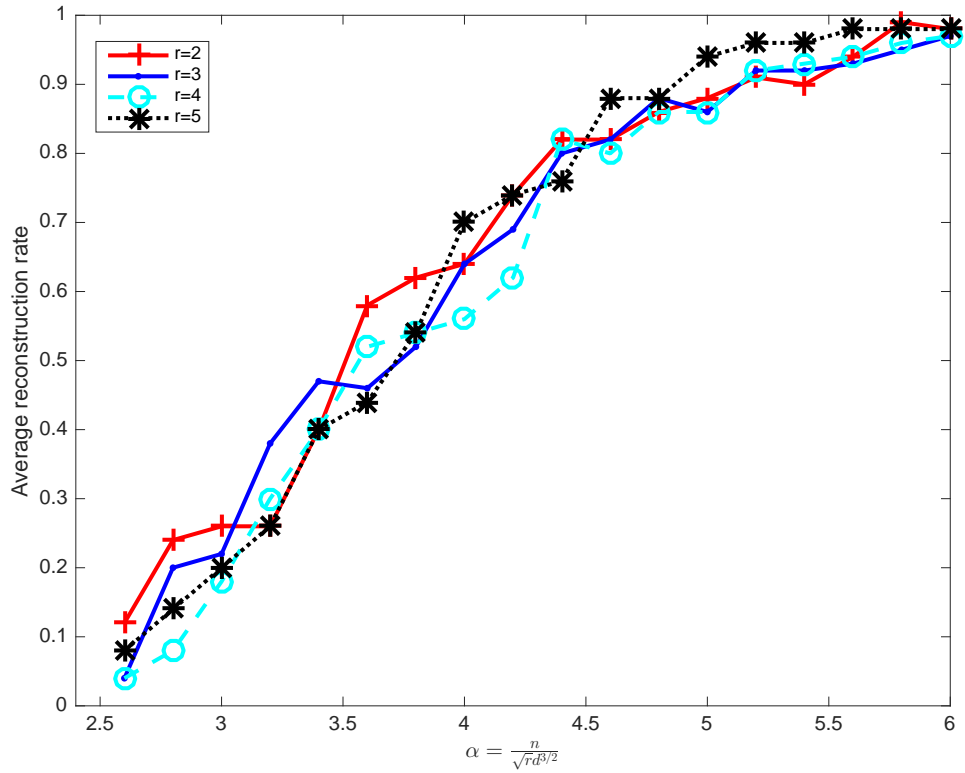}
\caption{Average reconstruction rate of the proposed approach when $d=50$ and $r=2,3,4,5$. For each $r$ and $\alpha$, the algorithm is repeated for $50$ times.}
\label{fig:fig1}
\end{figure}
\begin{figure}
\centering
\includegraphics[width=15cm]{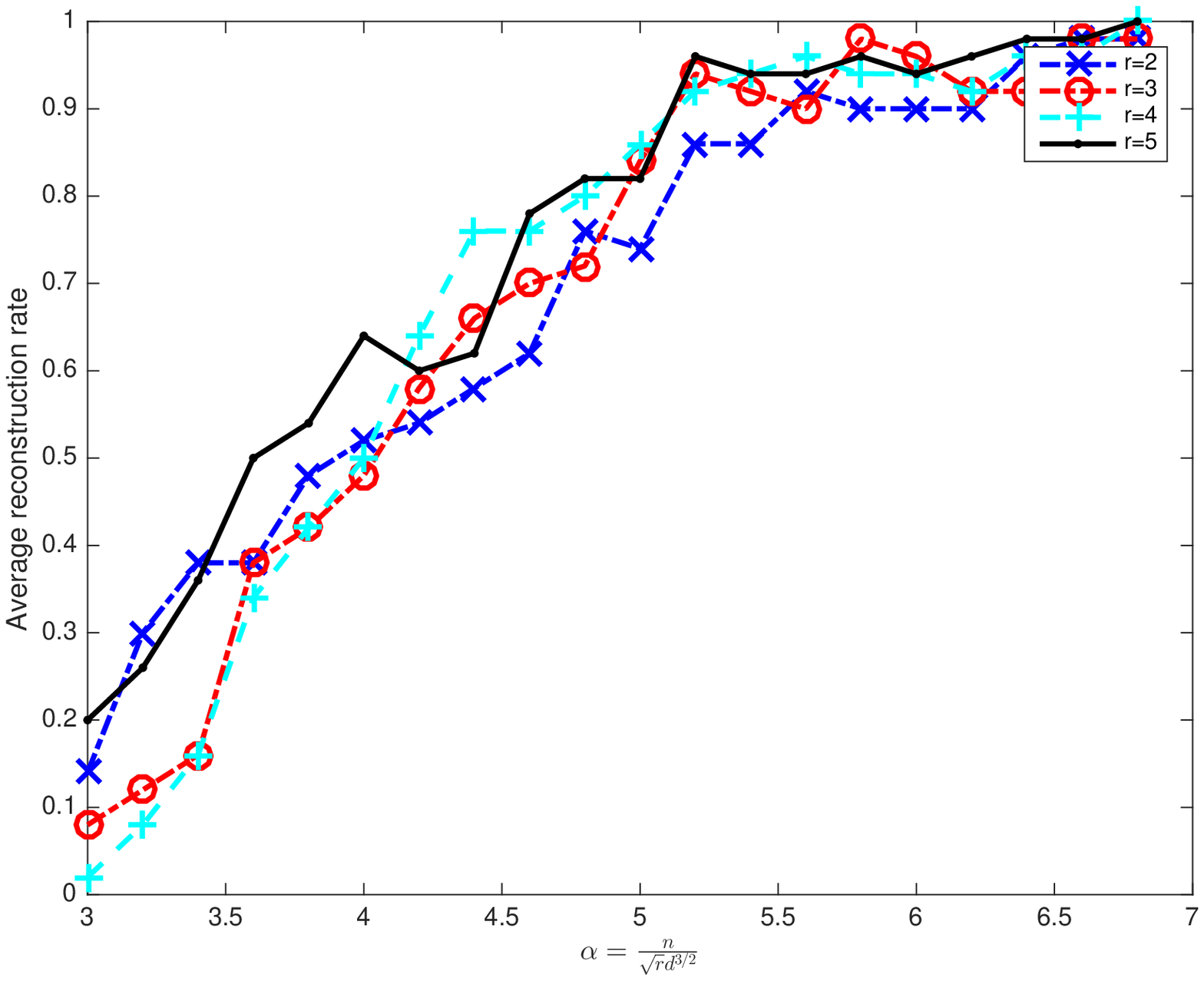}
\caption{Average reconstruction rate of the proposed approach when $d=100$ and $r=2,3,4,5$. For each $r$ and $\alpha$, the algorithm is repeated for $50$ times.}
\label{fig:fig2}
\end{figure}

As mentioned before, in addition to the gradient descent algorithm described in Section~\ref{sec:Flocal}, several other algorithms can also be applied to optimize $F(\bX,\bY,\bZ)$ locally. A notable example is the geometrical conjugate gradient descent algorithm on Riemannian manifolds proposed by \cite{kressner2014low}. Although we have focused on the analysis of the gradient descent algorithm, we believe similar results could also be established for these other algorithms as well. In essence, the success of these methods is determined by the quality of the initialization, which the method from Section \ref{sec:init} could be readily applied. We leave the more rigorous theoretical analysis for future work, we conducted a set of numerical experiments to illustrate the similarity between these optimization algorithms while highlighting the crucial role of initialization.

We considered a similar setup as before with $d=50$,$r=5$ and $d=100, r=3$. We shall refer to our method as GoG and the geometrical conjugate gradient descent algorithm as GeoCG, for brevity. Note that the GeoCG algorithm was proposed without considering the theoretical requirement on the sample size and the algorithm is initiated with a random guess. We first tested both algorithms with a reliable initialization as proposed in Section \ref{sec:init}. That is, we started with $\hat\bU, \hat\bV, \hat\bW$ obtained from the spectral algorithm and let $\hat{\bC}\in\mathbb{R}^{r\times r\times r}$ be the minimizer of (\ref{eq:opt}). Then, the GeoCG(Spectral) algorithm is initialized from the starting point $\hat\bA^{(0)}=(\hat\bU,\hat\bV,\hat\bW)\cdot \hat\bC$. For each $\alpha=\frac{n}{\sqrt{r}d^{3/2}}$, the GeoCG algorithm is repeated for $50$ times.
The reconstruction rates are as shown in the Cyan curves in Figure~\ref{fig:fig3} and \ref{fig:fig4}. It is clear that both algorithms perform well and are comparable. 

\begin{figure}
\centering
\includegraphics[width=15cm]{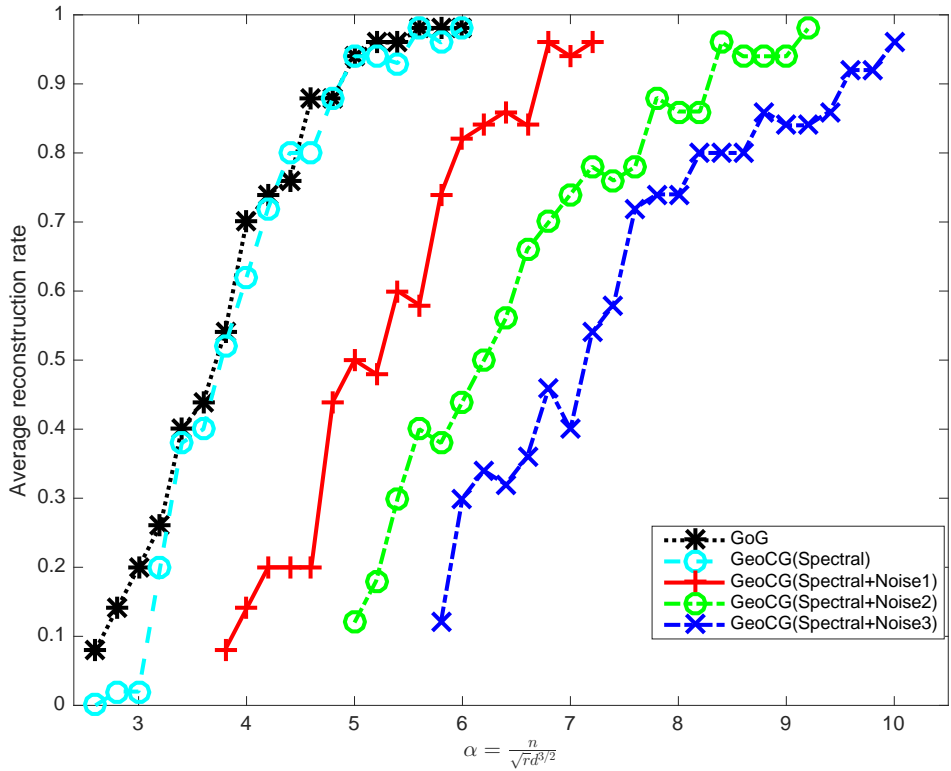}
\caption{Comparison between GoG and GeoCG algorithm when $d=50$ and $r=5$. The successful rates of GeoCG algorithm depend on the initialization. Here GeoCG(Spectral) means that the GeoCG algorithm is initialized with the spectral methods as GoG algorithm. The black and Cyan curves show that GoG and GeoCG algorithm perform similarly when both are initialized with spectral methods. Here GeoCG(Spectral+Noise$X$) means that GeoCG algorithm is initialized with spectral methods plus random perturbation. If $X$ is larger, the perturbation is larger and the initialization is further away from the truth, in which case the reconstruction rate decreases.
}
\label{fig:fig3}
\end{figure}
\begin{figure}
\centering
\includegraphics[width=15cm]{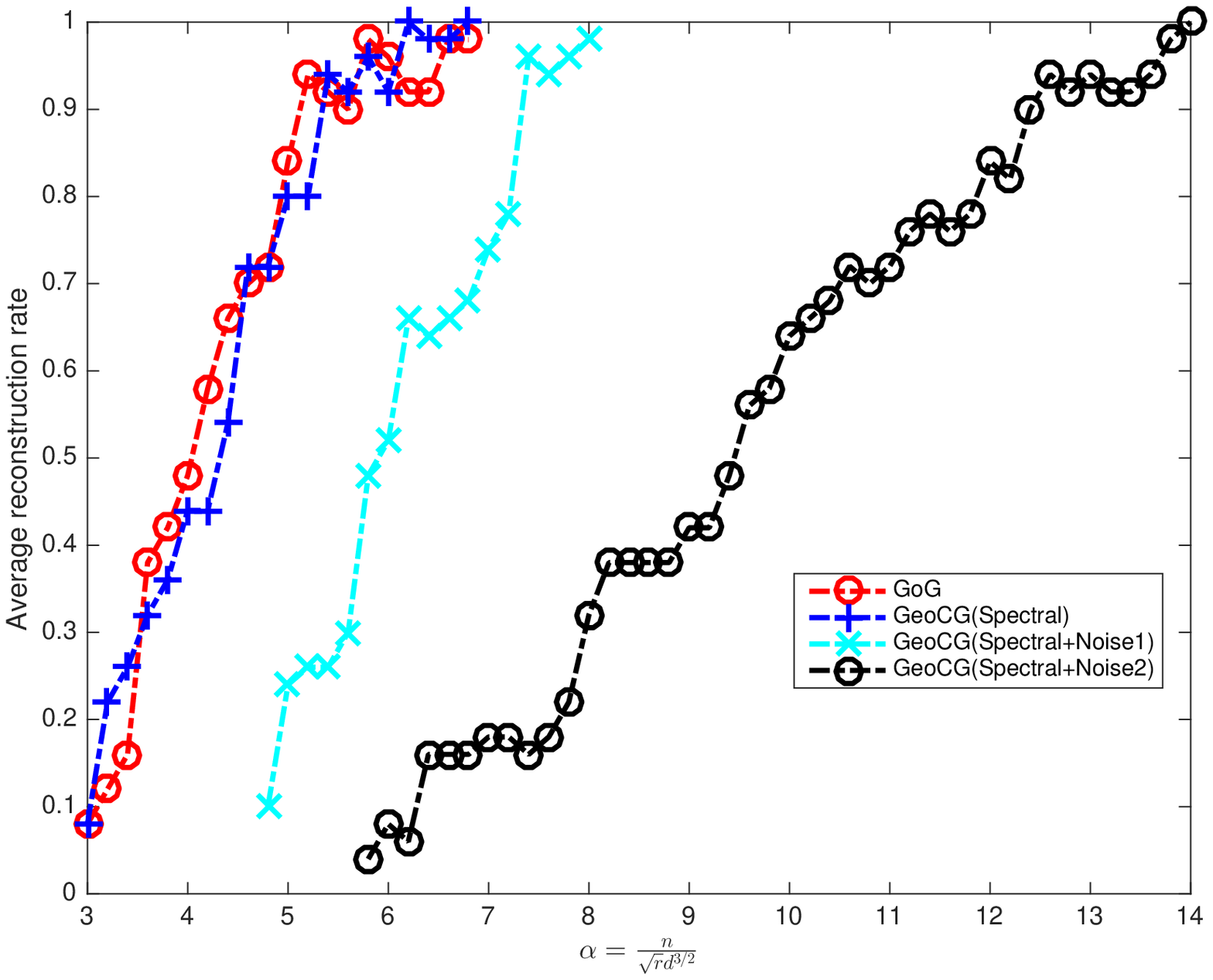}
\caption{Comparison between GoG and GeoCG algorithm when $d=100$ and $r=3$. The successful rates of GeoCG algorithm depend on the initialization.}
\label{fig:fig4}
\end{figure}

To illustrate that successful recovery hinges upon the appropriate initialization, we now consider applying GeoCG algorithm with a randomly perturbed spectral initialization. More specifically, the GeoCG algorithm is initialized with $\hat\bA^{(0)}+\sigma \bZ$ where $\bZ\in\mathbb{R}^{d\times d\times d}$ is a random tensor with i.i.d. standard normal entries and $\sigma>0$ represents the noise level. Figure~\ref{fig:fig3} and \ref{fig:fig4} show that the reconstruction rate decreases when $\sigma$ gets larger.

These observations confirm the insights from our theoretical development: that the objective function $F$ is well-behaved locally and therefore with appropriate initialization can lead to successful recovery of low rank tensors.

\section{Discussions}
\label{sec:disc}
In this paper, we proved that with $n\geq C\mu_0^3r_1r_2r_3(rd_1d_2d_3)^{1/2}\log^{7/2}(d)$ uniformly sampled entries, a tensor $\bT$ of multilinear ranks ($r_1,r_2,r_3$) can be recovered with high probability with a polynomial time algorithm. In doing so, we argue that the underlying optimization problem is well behaved in a neighborhood around the truth and therefore, the sample size requirement is largely driven by our ability to initialize the algorithm appropriately. To this end, a new spectral method based on estimating the second moment of tensor unfoldings is proposed. In the low rank case, e.g., $r=O(1)$, this sample size requirement is essentially of the same order as $d^{3/2}$, up to a polynomial of $\log d$ term. This matches the sample size requirement for nuclear norm minimization which is NP hard to compute in general. An argument put forth by \cite{barak2016noisy} suggests that such a dependence on the dimension may be optimal for polynomial time algorithms unless a more efficient algorithm exists for the 3-SAT problem.

Even though our framework is established for third order tensors, it can be naturally extended to higher order tensors. Indeed, to complete a $k$th order tensor $\bT\in\mathbb{R}^{d\times d\times\ldots\times d}$ with multilinear ranks $(r,r,\ldots, r)$, we can apply similar type of algorithms for optimizing over product of Grassmanianns. In order to ensure exact recovery, we can start with similar initialization where we unfold the tensor to $d\times d^{k-1}$ matrices. Following an identical argument, it can be derived in the same fashion that the sample size requirement for exact recovery now becomes
$$
n\geq Cd^{k/2}{\rm polylog}(r,\log d)
$$
for some constant $C>0$. Unlike the third order case, the dependence on the dimensionality ($d^{k/2}$) is worse than the nuclear norm minimization ($d^{3/2}$) for $k>3$. See \cite{yuan2016incoherent}. In general, it remains unclear whether the requirement of $d^{k/2}$ is the best attainable for polynomial time algorithms for $k>3$.

\section{Proofs}
\label{sec:proof}

Throughout the proofs, we shall use $C$ and similarly $C_1, C_2$ and etc. to denote generic numerical positive constants that may take different values at each appearance.

\begin{proof}[Proof of Theorem~\ref{thm:main}]
In view of Theorem~\ref{th:GoGconv}, the proof of Theorem~\ref{thm:main} is immediate if we are able to find an initial point $(\bX^{(0)},\bY^{(0)},\bZ^{(0)})\in\mathcal{N}\big(C(\alpha\kappa_0^2\log d)^{-1},3\mu_0)$. Clearly, under the conditions on $n$ given in Theorem~\ref{thm:main}, the spectral initialization $(\hat\bU,\hat\bV,\hat\bW)$ satisifies that
$$
d_{\rm p}\Big((\hat\bU,\hat\bV,\hat\bW),(\bU,\bV,\bW)\Big)\leq C(\alpha\kappa_0^2\log d)^{-1}
$$
with probability at least $1-3d^{-\alpha}$. It remains to show that we can derive an incoherent initial value from $(\hat\bU,\hat\bV,\hat\bW)$ in polynomial time, which is an immediate consequence of the following lemma due to \cite{keshavan2009matrix}.% ([Remark~6.2]).
\begin{lemma}\label{lem:incoherent}
Let $\hat\bU,\bU\in\mathbb{R}^{d\times r}$ with $\hat\bU^\top\hat\bU=\bU^\top\bU=\bI_r$ and $\mu(\bU)\leq \mu_0$. If $d_{\rm p}(\hat\bU,\bU)\leq \delta\leq \frac{1}{16}$, then there exists an algorithm on $\hat\bU$ whose complexity is $O(dr^2)$ which produces a candidate $\tilde\bU\in\mathcal{G}(d,r)$ such that $\mu(\tilde\bU)\leq 3\mu_0$ and $d_{\rm p}(\tilde\bU,\bU)\leq 4\delta.$
\end{lemma}
By applying the algorithm claimed in Lemma~\ref{lem:incoherent} onto $\hat\bU,\hat\bV,\hat\bW$, we obtain $(\bX^{(0)},\bY^{(0)},\bZ^{(0)})=(\tilde\bU,\tilde\bV,\tilde\bW)\in\mathcal{N}\big(C(\alpha\kappa_0^2\log d)^{-1},3\mu_0)$, which concludes the proof of Theorem~\ref{thm:main}.
\end{proof}

\medskip
\begin{proof}[Proof of Theorem \ref{th:spectral}]
Note that $\hat{\bN}$ is actually U-statistics. Using a standard decoupling technique for U-statistics \citep[see, e.g.,][]{de1995decoupling, de1999decoupling}, we get
$$
\PP(\|\hat{\bN}-\bN\|> t)\le 15\PP(15\|\tilde{\bN}-\bN\|> t)
$$
for any $t>0$, where
$$
\tilde{\bN}:={1\over 2n(n-1)}\sum_{i\neq j} (\bX_i\bY_j^\top+\bY_j\bX_i^\top),
$$
and $\{\bY_i: 1\le i \le n\}$ is an independent copy of $\{\bX_i: 1\le i\le n\}$. We shall then focus, in what follows, on bounding $\PP(\|\tilde{\bN}-\bN\|> t)$.

Observe that
$$
\tilde{\bN}={1\over 2n(n-1)} (\bS_{1n}\bS_{2n}^\top+\bS_{2n}\bS_{1n}^\top)-{1\over 2n(n-1)}\sum_{i=1}^n (\bX_i\bY_i^\top+\bY_i\bX_i^\top),
$$
where
$$
\bS_{1k}=\sum_{i=1}^k \bX_i, \qquad {\rm and}\qquad \bS_{2k}=\sum_{i=1}^k \bY_i.
$$
An application of Chernoff bound yields that, with probability at least $1-m^{-\alpha}$,
$$
\|\bS_{1n}\|_{\ell_\infty}\le (3\alpha+7)m_1m_2\|\bM\|_{\max}\left({n\over m_2}+\log m\right)
$$
for any $\alpha>0$, where
$$
\|\bS_{1n}\|_{\ell_\infty}:=\max_{1\le j\le m_2}\sum_{1\le i\le m_1} \left|(\bS_{1n})_{ij}\right|.
$$
See, e.g., \cite{yuan2016incoherent}. Denote this event by $\calE_1$. On the other hand, as shown by \cite{recht2011simpler}, with probability at least $1-2m^{-\alpha}$,
$$
\left\|{1\over n}\bS_{1n}-\bM\right\|\le \sqrt{8(\alpha+1) m_1m_2m\log m\over 3n}\|\bM\|_{\max},
$$
and
$$
\left\|{1\over n}\bS_{2n}-\bM\right\|\le \sqrt{8(\alpha+1) m_1m_2m\log m\over 3n}\|\bM\|_{\max}.
$$
Denote this event by $\calE_2$. Write $\calE=\calE_1\cap \calE_2$. It is not hard to see that for any $t\ge 0$,
$$
\PP\left\{\left\|\tilde{\bN}-\bN\right\|> t\right\}\le \PP\left\{\left\|\tilde{\bN}-\bN\right\|> t\bigcap\calE\right\}+3m^{-\alpha}
$$
We shall now proceed to bound the first probability on the right hand side.

Write
\begin{eqnarray*}
\tilde{\bN}-\bN&=&{1\over 2n(n-1)} \left[\left(\bS_{1n}-n\bM\right)(\bS_{2n}-n\bM)^\top+(\bS_{2n}-n\bM)\left(\bS_{1n}-n\bM\right)^\top\right]\\
&&\hskip 50pt+{1\over 2(n-1)}\left[\bM\left(\bS_{2n}-n\bM\right)^\top+\left(\bS_{2n}-n\bM\right)\bM^\top\right]\\
&&\hskip 50pt+{1\over 2(n-1)}\left[\bM\left(\bS_{1n}-n\bM\right)^\top+\left(\bS_{1n}-n\bM\right)\bM^\top\right]\\
&&\hskip 50pt-{1\over 2n(n-1)}\sum_{i=1}^n (\bX_i\bY_i^\top+\bY_i\bX_i^\top-2\bM\bM^\top)\\
&=:&\bA_1+\bA_2+\bA_3+\bA_4.
\end{eqnarray*}
We bound each of the four terms on the rightmost hand side separately. For brevity, write
$$
\Delta_{1k}=\bS_{1k}-k\bM,\qquad {\rm and}\qquad \Delta_{2k}=\bS_{2k}-k\bM.
$$

We begin with
$$
\bA_1={1\over 2n(n-1)} \left(\Delta_{1n}\Delta_{2n}^\top+\Delta_{2n}\Delta_{1n}^\top\right).
$$
By Markov inequality, for any $\lambda>0$,
$$
\PP\left\{\left\|\bA_1\right\|> t\bigcap \calE\right\}\le \PP\left\{\tr \exp\left(\lambda \bA_1\right)>\exp(\lambda t)\bigcap \calE\right\}\le e^{-\lambda t}\EE\tr \exp\left[\lambda \bA_1{\bf 1}_\calE\right].
$$
Repeated use of Golden-Thompson inequality yields,
\begin{eqnarray*}
\EE\tr \exp\left[\lambda \bA_1{\bf 1}_\calE\right]&=&\EE\left(\EE\left\{\tr \exp\left[\lambda \bA_1\right]{\bf 1}_\calE\biggr| \bS_{1n}\right\}\right)\\
&\le&\EE\biggr(\EE\left\{\tr \exp\left[{\lambda\over 2n(n-1)} (\Delta_{1n}\Delta_{2,n-1}^\top+\Delta_{2,n-1}\Delta_{1n}^\top)\right]{\bf 1}_\calE\biggr|\bS_{1n}\right\}\times\\
&&\hskip 20pt \left\|\EE\left\{\exp\left[{\lambda\over 2n(n-1)} (\Delta_{1n}(\bY_n-\bM)^\top+(\bY_n-\bM)\Delta_{1n}^\top)\right]{\bf 1}_\calE\biggr|\bS_{1n}\right\}\right\|\biggl)\\
&\le&\cdots\cdots\\
&\le&\EE\left(\left\|\EE\left\{\exp\left[{\lambda\over 2n(n-1)} (\Delta_{1n}(\bY_n-\bM)^\top+(\bY_n-\bM)\Delta_{1n}^\top)\right]{\bf 1}_\calE\biggr|\bS_{1n}\right\}\right\|^n\right)
\end{eqnarray*}
By triangular inequality,
$$
\left\|{\lambda\over 2n(n-1)}\left[\Delta_{1n}(\bY_n-\bM)^\top+(\bY_n-\bM)\Delta_{1n}^\top\right]\right\|\le {\lambda\over n(n-1)}\left(\|\Delta_{1n}\bY_n^\top\|+\|\Delta_{1n}\bM^\top\|\right).
$$
Under the even $\calE_1$, 
\begin{eqnarray*}
\|\Delta_{1n}\bY_n^\top\|&\le& \|\bS_{1n}\bY_n^\top\|+n\|\bM\bY_n^\top\|\\
&\le& (3\alpha+7)m_1^2m_2^2\|\bM\|^2_{\max}\left({n\over m_2}+\log m\right)+nm_1m_2\|\bM\|_{\max}\|\bM\|.
\end{eqnarray*}
On the  other hand, under the event $\calE_2$,
$$
\|\Delta_{1n}\bM^\top\|\le\|\Delta_{1n}\|\|\bM\|\le\sqrt{{8\over 3}n(\alpha+1) m_1m_2m\log m}\|\bM\|_{\max}\|\bM\|
$$
Recall that
$$
n\cdot\min\{m_1, m_2\}\ge {8\over 3}(\alpha+1) \log m.
$$
Then
\begin{eqnarray*}
&&\left\|{\lambda\over 2n(n-1)}\left[\Delta_{1n}(\bY_n-\bM)^\top+(\bY_n-\bM)\Delta_{1n}^\top\right]\right\|\\
&\le&{\lambda\over n(n-1)}\left( (3\alpha+7)m_1^2m_2^2\|\bM\|^2_{\max}\left({n\over m_2}+\log m\right)+nm_1m_2\|\bM\|_{\max}\|\bM\|\right).
\end{eqnarray*}
Therefore, for any
$$
\lambda\le n(n-1)\left( (3\alpha+7)m_1^2m_2^2\|\bM\|^2_{\max}\left({n\over m_2}+\log m\right)+nm_1m_2\|\bM\|_{\max}\|\bM\|\right)^{-1},
$$
we get
\begin{eqnarray*}
&&\EE\left\{\exp\left[{\lambda\over 2n(n-1)}\left[\Delta_{1n}(\bY_n-\bM)^\top+(\bY_n-\bM)\Delta_{1n}^\top\right]\right]{\bf 1}_\calE\biggr|\bS_{1n}\right\}\\
&\preceq& \bI_{m_1}+\EE\left\{\left[{\lambda\over 2n(n-1)}\left[\Delta_{1n}(\bY_n-\bM)^\top+(\bY_n-\bM)\Delta_{1n}^\top\right]\right]^2{\bf 1}_\calE\biggr|\bS_{1n}\right\}\\
&\preceq& \bI_{m_1}+\EE\left\{\left[{\lambda\over 2n(n-1)}\left(\Delta_{1n}\bY_n^\top+\bY_n\Delta_{1n}^\top\right)\right]^2{\bf 1}_\calE\biggr|\bS_{1n}\right\}\\
&\preceq&\bI_{m_1}+{\lambda^2m_1m_2\|\bM\|_{\max}^2\over 4n^2(n-1)^2}\left[(m_1+2)\Delta_{1n}\Delta_{1n}^\top+\tr(\Delta_{1n}\Delta_{1n}^\top)\bI_{m_1}\right]{\bf 1}_\calE
\end{eqnarray*}
where in the first inequality, we used the facts that
$$
\exp(\bA)\le \bI_d+\bA+\bA^2
$$
for any $\bA\in \RR^{d\times d}$ such that $\|\bA\|\le 1$, and
$$
\EE\left\{\left[{\lambda\over 2n(n-1)}\left[\Delta_{1n}(\bY_n-\bM)^\top+(\bY_n-\bM)\Delta_{1n}^\top\right]\right]{\bf 1}_\calE\biggr|\bS_{1n}\right\}=0.
$$

Recall that
$$
\tr(\Delta_{1n}\Delta_{1n}^\top)\le m_1\|\Delta_{1n}\Delta_{1n}^\top\|.
$$
This implies that
\begin{eqnarray*}
&&\left\|\EE\left\{\exp\left[{\lambda\over 2n(n-1)}\left[\Delta_{1n}(\bY_n-\bM)^\top+(\bY_n-\bM)\Delta_{1n}^\top\right]\right]{\bf 1}_\calE\biggr|\bS_{1n}\right\}\right\|\\
&\le& 1+{\lambda^2\|\bM\|_{\max}^2m_1^2m_2\over 2n^2(n-1)^2}\|\Delta_{1n}\Delta_{1n}^\top\|{\bf 1}_\calE\\
&\le& 1+{8(\alpha+1)\lambda^2\|\bM\|_{\max}^4m_1^3m_2^2m\log m\over 3n(n-1)^2},
\end{eqnarray*}
where the last inequality follows from the definition of $\calE_2$. Thus,
\begin{eqnarray*}
\EE\tr \exp\left[\lambda \bA_1{\bf 1}_\calE\right]%&\le&\exp\left[{n\lambda^2\|\bM\|_{\max}^2m_1^2m_2\over (n-1)^2}\left(\|\bM\|+\sqrt{8(\alpha+1) m_1m_2^2\log m_2\over 3n}\|\bM\|_{\max}\right)^2\right]\\
&\le&\exp\left[\lambda^2%\left({2n\|\bM\|_{\max}^2\|\bM\|^2m_1^2m_2\over (n-1)^2}+
{16(\alpha+1)\|\bM\|_{\max}^4m_1^3m_2^2m\log m\over 3(n-1)^2}\right].%\\
%&\le&\exp\left[\lambda^2\left({2n\|\bM\|_{\max}^4m_1^3m_2^2\over (n-1)^2}+{16(\alpha+1)\|\bM\|_{\max}^4m_1^3m_2^3\log m_2\over 3(n-1)^2}\right)\right].
\end{eqnarray*}
Taking
\begin{eqnarray*}
\lambda&=&\min\biggl\{%{(n-1)^2t\over 8n\|\bM\|_{\max}^2\|\bM\|^2m_1^2m_2},
{3(n-1)^2t\over 64(\alpha+1)\|\bM\|_{\max}^4m_1^3m_2^2m\log m},\\&&\hskip 20pt{n(n-1)\over (6\alpha+14)m_1^2m_2^2\|\bM\|^2_{\max}\left({n/m_2}+\log m\right)},{n(n-1)\over 2nm_1m_2\|\bM\|_{\max}\|\bM\|}\biggr\}
\end{eqnarray*}
yields
\begin{eqnarray*}
\PP\left\{\left\|\bA_1\right\|> t\bigcap \calE\right\}&\le&\exp\biggl(-\min\biggl\{%{(n-1)^2t^2\over 16n\|\bM\|_{\max}^2\|\bM\|^2m_1^2m_2},
{3(n-1)^2t^2\over 128(\alpha+1)\|\bM\|_{\max}^4m_1^3m_2^2m\log m},\\
&&\hskip 20pt{n(n-1)t\over (12\alpha+28)m_1^2m_2^2\|\bM\|^2_{\max}\left({n/m_2}+\log m\right)},{n(n-1)t\over 4nm_1m_2\|\bM\|_{\max}\|\bM\|}\biggr\}\biggr)
\end{eqnarray*}

We now proceed to bound $\bA_2$ and $\bA_3$. Both terms can be treated in an identical fashion and we shall consider only $\bA_2$ here to fix ideas. As before, it can be derived that
$$
\PP\left\{\left\|\bA_2\right\|> t\bigcap \calE\right\}\le \exp(-\lambda t)\left\|\EE\left\{\exp\left[{\lambda\over 2(n-1)} (\bM(\bY_n-\bM)^\top+(\bY_n-\bM)\bM^\top)\right]{\bf 1}_\calE\right\}\right\|^n.
$$
By taking
$$
\lambda\le {n-1\over \|\bM\|^2+m_1m_2\|\bM\|\|\bM\|_{\max}},
$$
we can ensure
\begin{eqnarray*}
\left\|{\lambda\over 2(n-1)} (\bM(\bY_n-\bM)^\top+(\bY_n-\bM)\bM^\top)\right\|&\le& {\lambda\over n-1}\left(\|\bM\|^2+m_1m_2\|\bM\|\|\bM\|_{\max}\right)\le 1.
\end{eqnarray*}
If this is the case, we can derive as before that
\begin{eqnarray*}
&&\left\|\EE\left\{\exp\left[{\lambda\over 2(n-1)} (\bM(\bY_n-\bM)^\top+(\bY_n-\bM)\bM^\top)\right]{\bf 1}_\calE\right\}\right\|\\
&\le& 1+\left\|\EE\left\{\left[{\lambda\over 2(n-1)} (\bM(\bY_n-\bM)^\top+(\bY_n-\bM)\bM^\top)\right]^2{\bf 1}_\calE\right\}\right\|\\
&\le& 1+\left\|\EE\left\{\left[{\lambda\over 2(n-1)} (\bM\bY_n^\top+\bY_n\bM^\top)\right]^2{\bf 1}_\calE\right\}\right\|\\
&\le& 1+{\lambda^2m_1^2m_2\|\bM\|_{\max}^2\|\bM\|^2\over 2(n-1)^2}.
\end{eqnarray*}
In particular, taking
$$
\lambda=\min\left\{{n-1\over 2\|\bM\|^2},{n-1 \over 2m_1m_2\|\bM\|\|\bM\|_{\max}},{(n-1)^2t\over nm_1^2m_2\|\bM\|_{\max}^2\|\bM\|^2}\right\}
$$
yields
$$
\PP\left\{\left\|\bA_2\right\|> t\bigcap \calE\right\}\le \exp\left(-\min\left\{{(n-1)t\over 4\|\bM\|^2}, {(n-1)t\over 2m_1m_2\|\bM\|\|\bM\|_{\max}}, {(n-1)^2t^2\over 2nm_1^2m_2\|\bM\|_{\max}^2\|\bM\|^2}\right\}\right).
$$

Finally, we treat $\bA_4$. Observe that
\begin{eqnarray*}
\|\bX_i\bY_i^\top+\bY_i\bX_i^\top-2\bM\bM^\top\|&\le& 2\|\bX_i\bY_i^\top\|+2\|\bM\|^2\\
&\le& 2m_1^2m_2^2\|\bM\|_{\max}^2+2\|\bM\|^2\\
&\le& 4m_1^2m_2^2\|\bM\|_{\max}^2,
\end{eqnarray*}
where the last inequality follows from the fact that $\|\bM\|\le \|\bM\|_{\rm F}\le \sqrt{m_1m_2}\|\bM\|_{\max}$. On the other hand
$$
\EE\left(\bX_i\bY_i^\top+\bY_i\bX_i^\top-2\bM\bM^\top\right)^2\preceq \EE\left(\bX_i\bY_i^\top+\bY_i\bX_i^\top\right)^2\preceq {2(m_1+1)m_1^2m_2^3\|\bM\|_{\max}^4} \bI.
$$
An application of matrix Bernstein inequality \citep{tropp2012user} yields
\begin{eqnarray*}
\PP\left\{\|\bA_4\|> t\cap \calE\right\}&\le& \PP\left\{\|\bA_4\|> t\right\}\\
&\le& m_1\exp\left(-{n^2(n-1)^2t^2/2\over 2n(m_1+1)m_1^2m_2^3\|\bM\|_{\max}^4+4m_1^2m_2^2\|\bM\|_{\max}^2t/3}\right).
\end{eqnarray*}

Putting the probability bounds for $\bA_1$, $\bA_2$, $\bA_3$, $\bA_4$ together, we have%, with probability at least $1-7m_2^{-\alpha}$,
$$
\PP\{\|\tilde{\bN}-\bN\|> t/15\}\le \sum_{k=1}^4\PP\{\|\bA_k\|> t/60\cap \calE\}+\PP\{\calE^c\}\le 7m^{-\alpha}
$$
by taking
$$
t=C\cdot\alpha^2\cdot{m_1^{3/2}m_2^{3/2}\log m\over n}\left[\left(1+\frac{m_1}{m_2}\right)^{1/2}+{m_1^{1/2}m_2^{1/2}\over n}+\left({n\over m_2\log m}\right)^{1/2}\right]\cdot \|\bM\|_{\max}^2,
$$
for some $C\geq 1680$. This immediately implies that
$$\mathbb{P}\big\{\|\hat\bN-\bN\|\geq t\big\}\leq 105m^{-\alpha}.$$
The proof is then concluded by replacing $\alpha$ with $\alpha+\log_m105$ and adjusting the constant $C$ accordingly.
%\begin{eqnarray*}
%\|\hat{\bN}-\bN\|
%%&\le& C\max\biggl\{\sqrt{\alpha m_1^2m_2\|\bM\|^2\|\bM\|^2_{\max}\log m_2\over n},\sqrt{\alpha^2\|\bM\|_{\max}^4m_1^3m_2^3\log^2m_2\over n^2},\\
%%&&\hskip 20pt {\alpha^2m_1^2m_2^2\|\bM\|_{\max}^2(n/m_2+\log m_2)\log m_2\over n^2},{\alpha m_1m_2\|\bM\|\|\bM\|_{\max}\log m_2\over n},\\
%%&&\hskip 20pt \sqrt{\alpha m_1^2m_2\|\bM\|^2\|\bM\|_{\max}^2\log m_2\over n}, {\alpha \|\bM\|^2\log m_2\over n},\\
%%&&\hskip 20pt \sqrt{\alpha m_1^3m_2^3\|\bM\|_{\max}^4\log m_2\over n^3},{\alpha m_1^2m_2^2\|\bM\|_{\max}^2\log m_2\over n^4}\biggr\}\\
%&\le& C\cdot \alpha^2\cdot{m_1^{3/2}m_2^{3/2}\log m_2\over n}\left[1+{m_1^{1/2}m_2^{1/2}\over n}+\left({n\over m_2\log m_2}\right)^{1/2}\right]\cdot \|\bM\|_{\max}^2
%\end{eqnarray*}
%for any $\alpha>1$.
\end{proof}
\medskip

\begin{proof}[Proof of Theorem \ref{thm:localF}]%Lemma~\ref{lem:firstorderlem}}\label{sec:lem:firstorderlem}
Let $\bP_{\bU}$, $\bP_\bV$ and $\bP_\bW$ be the projection matrices onto the column space of $\bU$, $\bV$ and $\bW$ respectively. Denote by $\bQ_\bT: \RR^{d_1\times d_2\times d_3}\to \RR^{d_1\times d_2\times d_3}$ a linear operator such that for any $\bA\in \RR^{d_1\times d_2\times d_3}$,
$$
\bQ_\bT\bA:=(\bP_\bU,\bP_\bV,\bP_\bW)\cdot \bA+(\bP_{\bU}^\perp, \bP_\bV,\bP_\bW)\cdot \bA+(\bP_\bU,\bP_{\bV}^\perp,\bP_\bW)\cdot\bA+(\bP_\bU,\bP_\bV,\bP_{\bW}^\perp)\cdot \bA,
$$
where $\bP_{\bU}^\perp=I-\bP_\bU$, and $\bP_\bV^\perp$ and $\bP_\bW^\perp$ are defined similarly. We shall also write $\bQ_\bT^\perp=\calI-\bQ_\bT$ where $\calI$ is the identity map. 

\paragraph{Basic facts about Grassmanianns.} Before proceeding, we shall first review some basic facts about the Grassmannians necessary for our proof. For further details, interested readers are referred to \cite{edelman1998geometry}. To fix ideas, we shall focus on $\bU\in\mathcal{G}(d_1,r_1)$. The tangent space of $\mathcal{G}(d_1,r_1)$ at $\bU$, denoted by $\calT_\bU\subset \mathbb{R}^{d_1\times r_1}$, can be identified with the property $\bU^\top\bD_\bU={\bf 0}$. The geodesic path from $\bU$ to another point $\bX\in\calG(d_1,r_1)$ with respect to the canonical Riemann metric can be explicitly expressed as:% (see \cite{edelman1998geometry} and \cite{keshavan2009matrix})
$$
\bX(t)=\bU\bR_\bU\cos({\bf \Theta}_\bU t)\bR_\bU^\top+\bL_\bU\sin({\bf \Theta_\bU}t)\bR_\bU^\top, \quad 0\leq t\leq 1
$$
for some $\bD_{\bU}\in \calT_\bU$ and $\bD_\bU=\bL_\bU{\bf \Theta}_\bU\bR_\bU^\top$ is its thin singular value decomposition. We can identify $\bX(0)=\bU$ and $\bX(1)=\bX$. The diagonal element of ${\bf \Theta}_\bU$ lie in $[-\pi/2, \pi/2]$ and can be viewed as the principle angles between $\bU$ and $\bX$.

It is easy to check
$$
d_{\rm p}(\bU,\bX)=\|\sin{\bf \Theta}_\bU\|_{\rm F}\quad {\rm and}\quad \|{\bf \Delta}_\bX\|_{\rm F}=\|\bU-\bX\|_{\rm F}=2\|\sin{\bf \Theta}_\bU/2\|_{\rm F}.
$$
Note that for any $\theta\in[0,\pi/2]$,
$$
\frac{\theta}{2}\leq\sqrt{2}\sin(\theta/2)\leq \sin\theta\leq 2\sin(\theta/2)\leq \theta.
$$
This implies that
$$d_{\rm p}(\bU,\bX)\leq \|\Delta_{\bX}\|_{\rm F}\leq \sqrt{2}d_{\rm p}(\bU,\bX).$$
Moreover, 
$$
\|\bU^\top{\bf \Delta}_{\bX}\|_{\rm F}=\|\cos({\bf \Theta}_{\bU})-\bI\|_{\rm F}=4\|\sin^2{\bf \Theta}_{\bU}/2\|_{\rm F}\leq 2\|\sin{\bf\Theta}_{\bU}\|_{\rm F}^2=2d_{\rm p}^2(\bU,\bX).
$$
With slight abuse of notation, write $\bD_\bX={d\bX(t)\over dt}\big|_{t=1}\in\calT_\bX$. $\bD_\bX$ can be more explicitly expressed as
$$
\bD_\bX=-\bU\bR_\bU{\bf \Theta_\bU}\sin{\bf \Theta_\bU}\bR_\bU^\top+\bL_\bU{\bf \Theta}_\bU\cos{\bf \Theta}_\bU\bR_\bU^\top.
$$
It is clear that 
$$\|\bD_\bX\|_{\rm F}^2=\|{\bf \Theta}_\bU\sin{\bf \Theta}_\bU\|_{\rm F}^2+\|{\bf \Theta}_\bU\cos{\bf \Theta}_\bU\|_{\rm F}^2=\|{\bf \Theta}_\bU\|_{\rm F}^2,$$ 
so that
$$
d_{\rm p}(\bU,\bX)\leq \|\bD_\bX\|_{\rm F}\leq 2d_{\rm p}(\bU,\bX).
$$
A couple of other useful relations can also be derived:
$$
\|\bD_\bX-{\bf \Delta}_\bX\|_{\rm F}^2=\|{\bf \Theta}_\bU\|_{\rm F}^2+4\|\sin({\bf \Theta}_\bU/2)\|_{\rm F}^2-2\langle{\bf \Theta}_\bU,\sin{\bf \Theta}_\bU\rangle\leq \|{\bf \Theta_\bU}-2\sin({\bf \Theta}_\bU/2)\|_{\rm F}^2\leq d_{\rm p}^4(\bU,\bX),
$$
and
$$\|\bU^\top \bD_\bX\|_{\rm F}=\|{\bf \Theta}_{\bU}\sin{\bf \Theta}_\bU\|_{\rm F}\leq 2\|\sin{\bf \Theta}_\bU\|_{\rm F}^2= 2d_{\rm p}^2(\bU,\bX).$$

\paragraph{Lower bound of the first statement.} 
Note that
\begin{equation}
\label{eq:lowerF}
F(\bX,\bY,\bZ)=\frac{1}{2}\big\|\mathcal{P}_{\Omega}\big(\hat{\bT}-\bT\big)\big\|_{\rm F}^2\geq \frac{1}{4}\big\|\mathcal{P}_{\Omega}\bQ_\bT(\hat{\bT}-\bT)\big\|_{\rm F}^2-\frac{1}{2}\big\|\mathcal{P}_{\Omega}\bQ_\bT^\perp(\hat{\bT})\big\|_{\rm F}^2,
\end{equation}
where
$$
\hat{\bT}=(\bX,\bY,\bZ)\cdot\bC
$$
and $\bC$ is given by \eqref{eq:defbC}. To derive the lower bound in the first statement, we shall lower bound $\|\mathcal{P}_{\Omega}\bQ_\bT(\hat{\bT}-\bT)\|_{\rm F}^2$ and upper bound $\|\mathcal{P}_{\Omega}\bQ_\bT^\perp(\hat{\bT})\|_{\rm F}^2$

By Lemma 5 of \cite{yuan2015tensor}, if $n\geq C_1\alpha\mu_0^2r^2d\log d$, then
$$
\mathbb{P}\left\{\left\|\bQ_\bT\left(\mathcal{I}-\frac{d_1d_2d_3}{n}\mathcal{P}_{\Omega}\right)\bQ_\bT\right\|\geq \frac{1}{2}\right\}\leq d^{-\alpha},
$$
where the operator norm is induced by the Frobenius norm, or the vectorized $\ell_2$ norm. Denote this event by $\calE_1$.  We shall now proceed under $\calE_1$. On event $\calE_1$,
$$
\|\mathcal{P}_{\Omega}\bQ_\bT(\hat{\bT}-\bT)\|_{\rm F}^2\geq \left\langle\mathcal{P}_{\Omega}\bQ_\bT(\hat{\bT}-\bT), \bQ_\bT(\hat\bT-\bT) \right\rangle\geq \frac{n}{2d_1d_2d_3}\|\bQ_\bT(\hat{\bT}-\bT)\|_{\rm F}^2.
$$
 Recall that 
\begin{equation}\label{eq:QThatTT}
\bQ_\bT(\hat{\bT}-\bT)=(\bU,\bV,\bW)\cdot(\bG-\bC)+({\bf \Delta}_{\bX},\bV,\bW)\cdot\bC+(\bU,{\bf\Delta}_{\bY},\bW)\cdot\bC+(\bU,\bV,{\bf \Delta}_{\bZ})\cdot\bC,
\end{equation}
where
$$
{\bf\Delta_X}:=\bX-\bU,\quad {\bf\Delta_Y}:=\bY-\bV,\quad {\rm and}\quad{\bf\Delta_Z}:=\bZ-\bW.
$$
Therefore,
\begin{align*}
\|\bQ_\bT(\hat{\bT}-\bT)\|_{\rm F}^2=&\|(\bU,\bV,\bW)\cdot(\bG-\bC)\|_{\rm F}^2+\|({\bf \Delta}_{\bX},\bV,\bW)\cdot\bC\|_{\rm F}^2+\|(\bU,{\bf\Delta}_{\bY},\bW)\cdot\bC\|_{\rm F}^2\\
&+\|(\bU,\bV,{\bf \Delta}_{\bZ})\cdot\bC\|_{\rm F}^2+2\langle(\bU,\bV,\bW)\cdot(\bG-\bC), ({\bf \Delta}_{\bX},\bV,\bW)\cdot\bC\rangle\\
&+2\langle(\bU,\bV,\bW)\cdot(\bG-\bC), (\bU,{\bf\Delta}_{\bY},\bW)\cdot\bC\rangle\\
&+2\langle(\bU,\bV,\bW)\cdot(\bG-\bC), (\bU,\bV,{\bf \Delta}_{\bZ})\cdot\bC\rangle\\
&+2\langle({\bf \Delta}_{\bX},\bV,\bW)\cdot\bC, (\bU,{\bf\Delta}_{\bY},\bW)\cdot\bC\rangle\\
&+2\langle({\bf \Delta}_{\bX},\bV,\bW)\cdot\bC, (\bU,\bV,{\bf \Delta}_{\bZ})\cdot\bC\rangle\\
&+2\langle(\bU,{\bf\Delta}_{\bY},\bW)\cdot\bC, (\bU,\bV,{\bf \Delta}_{\bZ})\cdot\bC\rangle.
\end{align*}
It is clear that
$$
\|(\bU,\bV,\bW)\cdot(\bG-\bC)\|_{\rm F}^2=\|\bG-\bC\|_{\rm F}^2.
$$
We now bound each of the remaining terms on the righthand side separately.

Note that
\begin{eqnarray*}
\|({\bf \Delta}_{\bX},\bV,\bW)\cdot\bC\|_{\rm F}^2&\geq& \frac{1}{2}\|({\bf \Delta}_{\bX},\bV,\bW)\cdot\bG\|_{\rm F}^2-\|({\bf \Delta}_{\bX},\bV,\bW)\cdot(\bC-\bG)\|_{\rm F}^2\\
&\ge&{1\over 2}\sigma_{\min}^2(\calM_1(\bG))\|{\bf \Delta}_{\bX}\|_{\rm F}^2-\sigma_{\max}^2(\calM_1(\bC-\bG))\|{\bf \Delta}_{\bX}\|_{\rm F}^2\\
&\ge&{1\over 2}\sigma_{\min}^2(\calM_1(\bG))\|{\bf \Delta}_{\bX}\|_{\rm F}^2-\|\bC-\bG\|_{\rm F}^2\|{\bf \Delta}_{\bX}\|_{\rm F}^2\\
&=&{1\over 2}\sigma_{\min}^2(\calM_1(\bT))\|{\bf \Delta}_{\bX}\|_{\rm F}^2-\|\bC-\bG\|_{\rm F}^2\|{\bf \Delta}_{\bX}\|_{\rm F}^2
\end{eqnarray*}
Similarly,
$$
\|(\bU,{\bf\Delta}_{\bY},\bW)\cdot\bC\|_{\rm F}^2\ge {1\over 2}\sigma_{\min}^2(\calM_2(\bT))\|{\bf \Delta}_{\bY}\|_{\rm F}^2-\|\bC-\bG\|_{\rm F}^2\|{\bf \Delta}_{\bY}\|_{\rm F}^2,
$$
and
$$
\|(\bU,\bV, {\bf\Delta}_{\bZ})\cdot\bC\|_{\rm F}^2\ge {1\over 2}\sigma_{\min}^2(\calM_3(\bT))\|{\bf \Delta}_{\bZ}\|_{\rm F}^2-\|\bC-\bG\|_{\rm F}^2\|{\bf \Delta}_{\bZ}\|_{\rm F}^2.
$$

On the other hand,
\begin{eqnarray*}
&&\left|\langle(\bU,\bV,\bW)\cdot(\bG-\bC), ({\bf \Delta}_{\bX},\bV,\bW)\cdot\bC\rangle\right|\\
&=&\left|\langle(\bU,\bV,\bW)\cdot(\bG-\bC), (\bP_\bU{\bf \Delta}_{\bX},\bV,\bW)\cdot\bC\rangle\right|\\
&\le&\left\|(\bU,\bV,\bW)\cdot(\bG-\bC)\right\|_{\rm F}\left\|(\bP_\bU{\bf \Delta}_{\bX},\bV,\bW)\cdot\bC\right\|_{\rm F}\\
&\le&\|\bG-\bC\|_{\rm F}\|\bP_\bU{\bf \Delta}_{\bX}\|_{\rm F}\|\bC\|\\
&\le& 2\|\bC\|\|\bG-\bC\|_{\rm F}d_{\rm p}^2(\bU,\bX).
\end{eqnarray*}
Observe that
$$
\|\bC\|\le \|\bG\|+\|\bG-\bC\|\le \|\bG\|+\|\bG-\bC\|_{\rm F}= \|\bT\|+\|\bG-\bC\|_{\rm F}.
$$
We get
\begin{eqnarray*}
\left|\langle(\bU,\bV,\bW)\cdot(\bG-\bC), ({\bf \Delta}_{\bX},\bV,\bW)\cdot\bC\rangle\right|\le 2\|\bT\|\|\bG-\bC\|_{\rm F}d_{\rm p}^2(\bU,\bX)\\
+2\|\bG-\bC\|_{\rm F}^2d_{\rm p}^2(\bU,\bX).
\end{eqnarray*}
Similarly,
\begin{eqnarray*}
\left|\langle(\bU,\bV,\bW)\cdot(\bG-\bC), (\bU, {\bf \Delta}_{\bY},\bW)\cdot\bC\rangle\right|\le 2\|\bT\|\|\bG-\bC\|_{\rm F}d_{\rm p}^2(\bV,\bY)\\
+2\|\bG-\bC\|_{\rm F}^2d_{\rm p}^2(\bV,\bY).
\end{eqnarray*}
and
\begin{eqnarray*}
\left|\langle(\bU,\bV,\bW)\cdot(\bG-\bC), (\bU,\bV,{\bf \Delta}_{\bZ})\cdot\bC\rangle\right|\le 2\|\bT\|\|\bG-\bC\|_{\rm F}d_{\rm p}^2(\bW,\bZ)\\
+2\|\bG-\bC\|_{\rm F}^2d_{\rm p}^2(\bW,\bZ).
\end{eqnarray*}

Finally, we note that
\begin{eqnarray*}
&&\left|\langle({\bf \Delta}_{\bX},\bV,\bW)\cdot\bC, (\bU,{\bf\Delta}_{\bY},\bW)\cdot\bC\rangle\right|\\
&=&\left|\langle(\bP_\bU{\bf \Delta}_{\bX},\bV,\bW)\cdot\bC, (\bU,\bP_\bV{\bf\Delta}_{\bY},\bW)\cdot\bC\rangle\right|\\
&\le&\|\bC\|^2\|\bP_\bU{\bf \Delta}_{\bX}\|_{\rm F}\|\bP_\bV{\bf \Delta}_{\bY}\|_{\rm F}\\
&\le&4\left(\|\bT\|+\|\bG-\bC\|_{\rm F}\right)^2d_{\rm p}^2(\bU,\bX)d_{\rm p}^2(\bV,\bY).
\end{eqnarray*}
And similarly,
$$
\left|\langle({\bf \Delta}_{\bX},\bV,\bW)\cdot\bC, (\bU,\bV, {\bf\Delta}_{\bZ})\cdot\bC\rangle\right|\le 4\left(\|\bT\|+\|\bG-\bC\|_{\rm F}\right)^2d_{\rm p}^2(\bU,\bX)d_{\rm p}^2(\bW,\bZ),
$$
and 
$$
\left|\langle(\bU,{\bf \Delta}_{\bY},\bW)\cdot\bC, (\bU,\bV, {\bf\Delta}_{\bZ})\cdot\bC\rangle\right|\le 4\left(\|\bT\|+\|\bG-\bC\|_{\rm F}\right)^2d_{\rm p}^2(\bV,\bY)d_{\rm p}^2(\bW,\bZ).
$$

Putting all these bounds together, we get
\begin{eqnarray*}
\|\bQ_\bT(\hat{\bT}-\bT)\|_{\rm F}^2\ge \|\bG-\bC\|_{\rm F}^2+\left({\Lambda_{\min}^2\over 2}-\|\bC-\bG\|_{\rm F}^2\right)\left(\|{\bf \Delta_X}\|_{\rm F}^2+\|{\bf \Delta_Y}\|_{\rm F}^2+\|{\bf \Delta_Z}\|_{\rm F}^2\right)\\
-4\|\bT\|\|\bG-\bC\|_{\rm F}d_{\rm p}^2((\bU,\bV,\bW),(\bX,\bY,\bZ))\\
-4\|\bG-\bC\|_{\rm F}^2d_{\rm p}^2((\bU,\bV,\bW),(\bX,\bY,\bZ))\\
-{8}\left(\|\bT\|+\|\bG-\bC\|_{\rm F}\right)^2d_{\rm p}^4((\bU,\bV,\bW),(\bX,\bY,\bZ)),
\end{eqnarray*}
where, with slight abuse of notation, we write
$$
\Lambda_{\min}:=\min\left\{\sigma_{\min}(\calM_1(\bT)),\sigma_{\min}(\calM_2(\bT)),\sigma_{\min}(\calM_3(\bT))\right\}.
$$
Recall that
$$\|{\bf \Delta}_\bX\|_{\rm F}\ge d_{\rm p}(\bX,\bU) ,\quad \|{\bf \Delta}_\bY\|_{\rm F}\ge d_{\rm p}(\bY,\bV),\quad{\rm and}\quad \|{\bf \Delta}_\bZ\|_{\rm F}\ge d_{\rm p}(\bZ,\bW),$$
so that
$$
\|{\bf \Delta_X}\|_{\rm F}^2+\|{\bf \Delta_Y}\|_{\rm F}^2+\|{\bf \Delta_Z}\|_{\rm F}^2\ge\frac{1}{3} d_{\rm p}^2((\bU,\bV,\bW),(\bX,\bY,\bZ)).
$$
We can further bound $\|\bQ_\bT(\hat{\bT}-\bT)\|_{\rm F}^2$ by
\begin{eqnarray*}
\|\bQ_\bT(\hat{\bT}-\bT)\|_{\rm F}^2&\ge& \|\bG-\bC\|_{\rm F}^2\\
&&+\left({\Lambda_{\min}^2\over 6}-5\|\bC-\bG\|_{\rm F}^2-4\|\bT\|\|\bG-\bC\|_{\rm F}\right)d_{\rm p}^2((\bU,\bV,\bW),(\bX,\bY,\bZ))\\
&&-{16}\left(\|\bT\|^2+\|\bG-\bC\|_{\rm F}^2\right)d_{\rm p}^4((\bU,\bV,\bW),(\bX,\bY,\bZ))
\end{eqnarray*}
Note that
$$
\Lambda_{\min}\ge \kappa_0^{-1}\Lambda_{\max}(\bT)\ge \kappa_0^{-1}\|\bT\|.
$$
If $d_{\rm p}\big((\bU,\bV,\bW),(\bX,\bY,\bZ)\big)\leq C(\alpha\kappa_0\log d)^{-1}$ for a sufficiently small $C$, we can ensure that
$$\|\bT\|d_{\rm p}\big((\bU,\bV,\bW),(\bX,\bY,\bZ)\big)\leq \frac{\Lambda_{\min}}{16}.$$
This implies that
$$
\|\bQ_\bT(\hat{\bT}-\bT)\|_{\rm F}^2\geq \frac{5}{8}\|\bG-\bC\|_{\rm F}^2+\Big(\frac{\Lambda_{\min}^2}{12}-4\|\bT\|\|\bG-\bC\|_{\rm F}\Big)d_{\rm p}^2\big((\bU,\bV,\bW),(\bX,\bY,\bZ)\big).
$$
We have thus proved that under the event $\calE_1$,
\begin{eqnarray}\nonumber
\|\calP_{\Omega}\bQ_\bT(\hat\bT-\bT)\|_{\rm F}^2&\geq& \frac{5n}{16d_1d_2d_3}\|\bG-\bC\|_{\rm F}^2\\
&&+\frac{n}{2d_1d_2d_3}\Big(\frac{\Lambda_{\min}^2}{12}-4\|\bT\|\|\bG-\bC\|_{\rm F}\Big)d_{\rm p}^2\big((\bU,\bV,\bW),(\bX,\bY,\bZ)\big).\label{eq:bdL1}
\end{eqnarray}

Now consider upper bounding $\|\mathcal{P}_{\Omega}\bQ_\bT^\perp\hat{\bT}\|_{\rm F}^2$. By Chernoff bound, it is easy to see that with probability $1-d^{-\alpha}$,
$$
\max_{\omega\in [d_1]\times[d_2]\times[d_3]} \sum_{i=1}^n \II(\omega_i=\omega)\le C\alpha\log d
$$
for some constant $C>0$. Denote this event by $\calE_2$. Under this event
$$
\|\mathcal{P}_{\Omega}\bQ_\bT^\perp\hat{\bT}\|_{\rm F}^2\le C(\alpha \log d)\left\langle \mathcal{P}_{\Omega}\bQ_\bT^\perp\hat{\bT},\bQ_\bT^\perp\hat{\bT}\right\rangle.
$$
To this end, it suffices to obtain upper bounds of 
$$
\left|\left\langle \mathcal{P}_{\Omega}\bQ_\bT^\perp\hat{\bT},\bQ_\bT^\perp\hat{\bT}\right\rangle\right|\leq \frac{n}{d_1d_2d_3}\|\bQ_\bT^\perp \hat\bT\|_{\rm F}^2+\left|\left\langle \mathcal{P}_{\Omega}\bQ_\bT^\perp\hat{\bT},\bQ_\bT^\perp\hat{\bT}\right\rangle-\frac{n}{d_1d_2d_3}\|\bQ_\bT^\perp \hat\bT\|_{\rm F}^2\right|.
$$
For $\gamma_1,\gamma_2>0$, define
$$
\mathcal{K}(\gamma_1,\gamma_2):=\big\{\bA\in\mathbb{R}^{d_1\times d_2\times d_3}: \|\bA\|_{\rm F}\leq 1, \|\bA\|_{\max}\leq\gamma_1, \|\bA\|_{\star}\leq\gamma_2\big\}.
$$
Consider the following empirical process:
$$
\beta_n({\gamma_1,\gamma_2}):=\underset{\bA\in \mathcal{K}(\gamma_1,\gamma_2)}{\sup}\left|\frac{1}{n}\left\langle \mathcal{P}_{\Omega}\bA,\bA\right\rangle-\frac{1}{d_1d_2d_3}\|\bA\|_{\rm F}^2\right|.
$$
Obviously,
$$
\left|\left\langle \mathcal{P}_{\Omega}\bQ_\bT^\perp\hat{\bT},\bQ_\bT^\perp\hat{\bT}\right\rangle\right|\leq \frac{n}{d_1d_2d_3}\|\bQ_\bT^\perp \hat\bT\|_{\rm F}^2+n\|\bQ_\bT^\perp\hat\bT\|_{\rm F}^2\beta_n\Big(\frac{\|\bQ_\bT^\perp\hat\bT\|_{\max}}{\|\bQ_\bT^\perp\hat\bT\|_{\rm F}}, \frac{\|\bQ_\bT^\perp\hat\bT\|_{\star}}{\|\bQ_\bT^\perp\hat\bT\|_F}\Big).
$$
We now appeal to the following lemma whose proof is given in Appendix~\ref{sec:emplem}.
\medskip
\begin{lemma}\label{emplem}
Given $0<\delta_1^-<\delta_1^+$, $0<\delta_2^-<\delta_2^+$ and $t\geq 1$, let
$$
\bar{t}=t+\log\big(\log_2(\delta_1^+/\delta_1^-)+\log_2(\delta_2^+/\delta_2^-)+3\big).
$$
Then exists a universal constant $C>0$ such that with probability at least $1-e^{-t}$, the following bound holds for all $\gamma_1\in[\delta_1^-,\delta_1^+]$ and all $\gamma_2\in[\delta_2^-,\delta_2^+]$
$$
\beta_n(\gamma_1,\gamma_2)\leq C\gamma_1\gamma_2\Big(\sqrt{\frac{d}{nd_1d_2d_3}}\log d+\frac{\log^{3/2}d}{n}\Big)+2\gamma_1\sqrt{\frac{\bar t}{nd_1d_2d_3}}+2\gamma_1^2\frac{\bar t}{n}
$$
\end{lemma}
\medskip

For any $\bA\in\mathbb{R}^{d_1\times d_2\times d_3}$, we have $\frac{\|\bA\|_{\max}}{\|\bA\|_{\rm F}}\in[1/d_1d_2d_3,1]$ and $\frac{\|\bA\|_{\star}}{\|\bA\|_{\rm F}}\in[1, d]$, we apply Lemma~\ref{emplem} with $\delta_1^-=\frac{1}{d_1d_2d_3}$, $\delta_1^+=1$, $\delta_2^-=1$ and $\delta_2^+=d$. By setting $t=\alpha\log d$  with ${\bar t}=t+\log\big(\log_2(d_1)+\log_2(d_2)+\log_2(d_3)+\log_2(d)+3\big)\leq 6\alpha\log d$, we obtain that with probability at least $1-d^{-\alpha}$, for all $\gamma_1\in [(d_1d_2d_3)^{-1}, 1]$ and $\gamma_2\in [1,d]$,
$$
\beta_n(\gamma_1,\gamma_2)\leq C_1\alpha\gamma_1\gamma_2\Big(\sqrt{\frac{d}{nd_1d_2d_3}}\log d+\frac{\log^{3/2}d}{n}\Big)+C_1\alpha\gamma_1\sqrt{\frac{\log d}{nd_1d_2d_3}}+C_1\alpha\gamma_1^2\frac{\log d}{n}.
$$
Denote this event by $\calE_3$. Under $\calE_3$, for any $\bA\in\mathbb{R}^{d_1\times d_2\times d_3}$,
\begin{align*}
\|\bA\|_{\rm F}^2\beta_n\Big(\frac{\|\bA\|_{\max}}{\|\bA\|_{\rm F}}, \frac{\|\bA\|_{\star}}{\|\bA\|_F}\Big)\leq &C_1\alpha\|\bA\|_{\max}\|\bA\|_{\star}\Big(\sqrt{\frac{d}{nd_1d_2d_3}}\log d+\frac{\log^{3/2}d}{n}\Big)\\
&+C_1 \alpha\|\bA\|_{\max}\|\bA\|_{\rm F}\sqrt{\frac{\log d}{nd_1d_2d_3}}+C_1\alpha\|\bA\|_{\max}^2\frac{\log d}{n}.
\end{align*}
This implies that
\begin{align}
\label{ineq:PAA}
\left\langle \mathcal{P}_{\Omega}\bA,\bA\right\rangle\leq \frac{n}{d_1d_2d_3}\|\bA\|_{\rm F}^2+C\alpha\|\bA\|_{\max}\|\bA\|_{\star}\Big(\sqrt{\frac{nd}{d_1d_2d_3}}\log d+\log^{3/2}d\Big).
\end{align}
We shall now focus on $\calE_3$ and obtain
\begin{align}\nonumber
\left\langle \mathcal{P}_{\Omega}\bQ_{\bT}^\perp\hat\bT,\bQ_{\bT}^\perp\hat\bT\right\rangle\leq &\frac{n}{d_1d_2d_3}\|\bQ_{\bT}^\perp\hat\bT\|_{\rm F}^2\\
&+C\alpha\|\bQ_{\bT}^\perp\hat\bT\|_{\max}\|\bQ_{\bT}^\perp\hat\bT\|_{\star}\Big(\sqrt{\frac{nd}{d_1d_2d_3}}\log d+\log^{3/2}d\Big).\label{ineq:L2}
\end{align}
%Recall the lower bound on $n$ imposed in Lemma~\ref{lem:firstorderlem}, the above bound can be simplifed as
%\begin{equation}
%\label{L2ineq1}
%\|\mathcal{P}_{\Omega}(\bL_2)\|_{\rm F}^2\leq \tau\|\bL_2\|_{\rm F}^2+C\alpha\|\bL_2\|_{\max}\|\bL_2\|_{\star}\sqrt{\frac{nd}{d_1d_2d_3}}\log d.
%\end{equation}
It remains to bound $\|\bQ_\bT^\perp\hat{\bT}\|_{\max}$, $\|\bQ_\bT^\perp\hat\bT\|_{\star}$ and $\|\bQ_\bT^\perp\hat{\bT}\|_{\rm F}$. Recall that
\begin{eqnarray*}
\bQ_\bT^\perp \hat\bT= (\bP_\bU^\perp \bX, \bP_\bV^\perp\bY,\bZ)\cdot\bC+(\bP_\bU^\perp \bX, \bY,\bP_\bW^\perp \bZ)\cdot\bC+(\bX,\bP_\bV^\perp\bY,\bP_\bW^\perp\bZ)\cdot\bC\\
+(\bP_\bU^\perp \bX,\bP_\bV^\perp\bY,\bP_\bW^\perp\bZ)\cdot\bC.
\end{eqnarray*}
Recall that $\Lambda_{\max}(\bC):=\max\{\|\calM_k(\bC)\|, k=1,2,3\}$. Clearly, $\Lambda_{\max}(\bC)\leq \Lambda_{\max}+\|\bG-\bC\|_{\rm F}$ where, with slight abuse of notation, we write $\Lambda_{\max}:=\Lambda_{\max}(\bT)$ for brevity. Then, 
\begin{eqnarray*}
\|\bQ_\bT^\perp\hat\bT\|_{\rm F}\leq \big(\Lambda_{\max}+\|\bG-\bC\|_{\rm F}\big)\Big(\|\bP_\bU^\perp\bX\|_{\rm F}\|\bP_\bV^\perp\bY\|_{\rm F}+\|\bP_\bU^\perp\bX\|_{\rm F}\|\bP_\bW^\perp\bZ\|_{\rm F}+\|\bP_\bW^\perp\bZ\|_{\rm F}\|\bP_\bV^\perp\bY\|_{\rm F}\Big)\\
+\big(\Lambda_{\max}+\|\bG-\bC\|_{\rm F}\big)\|\bP_\bU^\perp\bX\|_{\rm F}\|\bP_\bV^\perp\bY\|_{\rm F}\|\bP_\bW^\perp\bZ\|_{\rm F}.
\end{eqnarray*}
Observe that
$$\|\bP_\bU^\perp\bX\|_{\rm F}=\|\bP_\bU^\perp {\bf \Delta}_\bX\|_{\rm F}\leq \|{\bf \Delta}_\bX\|_{\rm F}\leq \sqrt{2}d_{\rm p}(\bU,\bX)$$ 
and
$$d_{\rm p}\big((\bU,\bV,\bW),(\bX,\bY,\bZ)\big)\leq (C\alpha\kappa_0\log d)^{-1}.$$
Therefore,
\begin{eqnarray*}
\|\bQ_\bT^\perp\hat\bT\|_{\rm F}&\leq& \big(\Lambda_{\max}+\|\bG-\bC\|_{\rm F}\big)\Big(2d_{\rm p}^2\big((\bU,\bV,\bW),(\bX,\bY,\bZ)\big)+2\sqrt{2}d_{\rm p}^3\big((\bU,\bV,\bW),(\bX,\bY,\bZ)\big)\Big)\\
&\leq& 3\big(\Lambda_{\max}+\|\bG-\bC\|_{\rm F}\big)d_{\rm p}^2\big((\bU,\bV,\bW),(\bX,\bY,\bZ)\big).
\end{eqnarray*}
It is clear that
$$\max_{k=1,2,3}\big\{\rank(\calM_k(\bQ_\bT^\perp\hat\bT))\big\}\leq 4r.$$
By Lemma~\ref{lemma:tensornorm},
$$
\|\bQ_\bT^\perp \hat\bT\|_{\star}\leq 4r\|\bQ_\bT^\perp\hat\bT\|_{\rm F}\leq 12r\big(\Lambda_{\max}+\|\bG-\bC\|_{\rm F}\big)d_{\rm p}^2\big((\bU,\bV,\bW),(\bX,\bY,\bZ)\big).
$$
%Similarly, $\|\bQ_\bT^\perp \hat\bT\|_{\max}$ can be bounded by two different terms (depending on the scale of $\tilde{\delta}$). 
Because of the incoherence condition
$$\max\{\mu({\bf \Delta}_{\bX}), \mu({\bf \Delta}_{\bY}), \mu({\bf \Delta}_{\bZ})\}\leq 9\mu_0,$$
we get
\begin{align*}
\|\bQ_\bT^\perp\hat\bT\|_{\max}\leq
% \Big\{12\big(\Lambda_{\max}+\|\bG-\bC\|_{\rm F}\big)\tilde{\delta}\mu_0\Big(\sqrt{\frac{r_1r_2}{d_1d_2}}+\sqrt{\frac{r_1r_3}{d_1d_3}}+\sqrt{\frac{r_2r_3}{d_2d_3}}\Big)\Big\}\\
54\big(\Lambda_{\max}+\|\bC-\bG\|_{\rm F}\big)\mu_0^{3/2}\sqrt{\frac{r_1r_2r_3}{d_1d_2d_3}}.
\end{align*}
By putting the bounds of $\|\bQ_\bT^\perp \hat\bT\|_{\rm F}, \|\bQ_\bT^\perp \hat\bT\|_{\max}$ and $\|\bQ_\bT^\perp\hat\bT\|_{\star}$ into (\ref{ineq:L2}), we conclude that on event $\calE_3$,
\begin{align}
\Big<\calP_\Omega\bQ_\bT^\perp\hat\bT,\bQ_\bT^\perp\hat\bT\Big>\leq& \frac{9n}{d_1d_2d_3}\big(\Lambda_{\max}+\|\bG-\bC\|_{\rm F}\big)^2d_{\rm p}^4\big((\bU,\bV,\bW),(\bX,\bY,\bZ)\big)\nonumber\\
&+C_1\bigg(\alpha r(\Lambda_{\max}+\|\bG-\bC\|_{\rm F})^2\mu_0^{3/2}\sqrt{\frac{r_1r_2r_3}{d_1d_2d_3}}\Big(\sqrt{\frac{nd}{d_1d_2d_3}}\log d+\log^{3/2}d\Big)\bigg)\nonumber\\
&\times d_{\rm p}^2\big((\bU,\bV,\bW),(\bX,\bY,\bZ)\big)\label{ineq:PQTperphatTQTperphatT}
\end{align}
for a universal constant $C_1>0$.
If $d_{\rm p}\big((\bU,\bV,\bW),(\bX,\bY,\bZ)\big)\leq (C_2\alpha\kappa_0\log d)^{-1}$ and
$$
n\geq C_2\Big(\alpha^4\mu_0^3\kappa_0^4 r^2r_1r_2r_3 d\log^4d+\alpha^2\mu_0^{3/2}\kappa_0^2 r(r_1r_2r_3d_1d_2d_3)^{1/2}\log^{5/2}d\Big).
$$
The above upper bound can be simplified as
\begin{eqnarray}\nonumber
\Big<\calP_\Omega\bQ_\bT^\perp\hat\bT,\bQ_\bT^\perp\hat\bT\Big>&\leq& \frac{n}{8C\alpha d_1d_2d_3\log d}\|\bG-\bC\|_{\rm F}^2\\
&&+\frac{n}{96C\alpha d_1d_2d_3\log d}\Lambda_{\min}^2 d_{\rm p}^2\big((\bU,\bV,\bW),(\bX,\bY,\bZ)\big).\label{ineq:bL2}
\end{eqnarray}
Therefore, under $\calE_2\cap\calE_3$,
\begin{equation}
\label{eq:bdL2}
\|\mathcal{P}_{\Omega}\bQ_\bT^\perp\hat{\bT}\|_{\rm F}^2\le \frac{n}{8d_1d_2d_3}\|\bG-\bC\|_{\rm F}^2+\frac{n}{96d_1d_2d_3}\Lambda_{\min}^2d_{\rm p}^2((\bU,\bV,\bW),(\bX,\bY,\bZ)).
\end{equation}
Combining \eqref{eq:lowerF},\eqref{eq:bdL1} and \eqref{eq:bdL2}, we conclude that
\begin{eqnarray}\nonumber
F(\bX,\bY,\bZ)&\ge& \frac{n}{64d_1d_2d_3}\|\bG-\bC\|_{\rm F}^2\\
&&+\frac{n}{d_1d_2d_3}\Big({\Lambda_{\min}^2\over 192}-\|\bT\|\|\bG-\bC\|_{\rm F}\Big)d_{\rm p}^2((\bU,\bV,\bW),(\bX,\bY,\bZ)),\label{ineq:Flower}
\end{eqnarray}
with probability at least
$$
\PP\{\calE_1\cap\calE_2\cap\calE_3\}\ge 1-3d^{-\alpha}.
$$

\paragraph{Upper bound of the first statement.} Let
$$
\tilde{\bT}=(\bX,\bY,\bZ)\cdot \bG.
$$
By definition of $\hat{\bT}$, 
$$
F(\bX,\bY,\bZ)={1\over 2}\|\calP_\Omega(\hat{\bT}-\bT)\|_{\rm F}^2\le {1\over 2}\|\calP_\Omega(\tilde{\bT}-\bT)\|_{\rm F}^2\le \|\calP_\Omega\bQ_\bT(\tilde{\bT}-\bT)\|_{\rm F}^2+\|\calP_\Omega\bQ_\bT^\perp\tilde{\bT}\|_{\rm F}^2
$$
Again, by Lemma 5 of \cite{yuan2015tensor}, on event $\calE_1\cap \calE_2$,
\begin{eqnarray*}
\|\calP_\Omega\bQ_\bT(\tilde{\bT}-\bT)\|_{\rm F}^2&\leq& C(\alpha \log d)\left\langle\calP_{\Omega}\bQ_\bT(\tilde\bT-\bT),\bQ_\bT(\tilde\bT-\bT) \right\rangle\\
&\le& {3C\alpha n\log d\over 2d_1d_2d_3}\|\bQ_\bT(\tilde{\bT}-\bT)\|_{\rm F}^2.
\end{eqnarray*}
Recall that 
$$
\bQ_\bT(\tilde{\bT}-\bT)=({\bf \Delta}_{\bX},\bV,\bW)\cdot\bG+(\bU,{\bf\Delta}_{\bY},\bW)\cdot\bG+(\bU,\bV,{\bf \Delta}_{\bZ})\cdot\bG.
$$
We have
\begin{eqnarray*}
\|\bQ_\bT(\tilde{\bT}-\bT)\|_{\rm F}^2&\le&3\left(\|({\bf \Delta}_{\bX},\bV,\bW)\cdot\bG\|_{\rm F}^2+\|(\bU,{\bf\Delta}_{\bY},\bW)\cdot\bG\|_{\rm F}^2+\|(\bU,\bV,{\bf \Delta}_{\bZ})\cdot\bG\|_{\rm F}^2\right).
\end{eqnarray*}
Note that
$$
\|({\bf \Delta}_{\bX},\bV,\bW)\cdot\bG\|_{\rm F}^2\le \sigma_{\max}^2(\calM_1(\bG))\|{\bf \Delta}_{\bX}\|_{\rm F}^2\le \Lambda_{\max}^2\|{\bf \Delta}_\bX\|_{\rm F}^2.
%\le \kappa_0^2\sigma_{\min}^2(\calM_1(\bG))\|{\bf \Delta}_{\bX}\|_{\rm F}^2\le\kappa_0^2\|\bT\|^2\|{\bf \Delta}_{\bX}\|_{\rm F}^2.
$$
Similar bounds hold for $\|(\bU,{\bf \Delta}_{\bY},\bW)\cdot\bG\|_{\rm F}^2$ and $\|(\bU,\bV,{\bf \Delta}_{\bZ})\cdot\bG\|_{\rm F}^2$.
We get on event $\calE_1\cap \calE_2$,
\begin{equation}
\label{eq:upperL1}
\|\calP_\Omega\bQ_\bT(\tilde{\bT}-\bT)\|_{\rm F}^2\le {9C\alpha n\log d\over d_1d_2d_3}\Lambda_{\max}^2d^2_{\rm p}((\bU,\bV,\bW),(\bX,\bY,\bZ)).
\end{equation}
On the other hand, following the same argument for bounding $\|\calP_\Omega\bQ_\bT^\perp\hat{\bT}\|_{\rm F}^2$ as in (\ref{eq:bdL2}), we can show that
\begin{eqnarray*}
\|\calP_\Omega\bQ_\bT^\perp \tilde{\bT}\|_{\rm F}^2\le C\alpha\log d\Big<\calP_{\Omega}\bQ_\bT^\perp\tilde{\bT}, \bQ_\bT^\perp\tilde\bT\Big>\le {n\over 96d_1d_2d_3}\Lambda_{\min}^2d^2_{\rm p}((\bU,\bV,\bW),(\bX,\bY,\bZ)),
\end{eqnarray*}
under the event $\calE_1\cap \calE_2\cap\calE_3$. In summary, we get on event $\calE_1\cap \calE_2\cap \calE_3$,
\begin{equation}\label{ineq:Fupper}
{d_1d_2d_3\over n}F(\bX,\bY,\bZ)\le 10C\alpha \Lambda_{\max}^2d^2_{\rm p}((\bU,\bV,\bW),(\bX,\bY,\bZ))\log d.
\end{equation}
The bounds (\ref{ineq:Flower}) and (\ref{ineq:Fupper}) imply that
\begin{eqnarray*}
\frac{n}{64d_1d_2d_3}\|\bG-\bC\|_{\rm F}^2+\frac{n}{d_1d_2d_3}\Big(\frac{\Lambda_{\min}^2}{192}-\|\bT\|\|\bG-\bC\|_{\rm F}\Big)d_{\rm p}^2\big((\bU,\bV,\bW),(\bX,\bY,\bZ)\big)\\
\leq F(\bX,\bY,\bZ)\leq \frac{10C\alpha n}{d_1d_2d_3}\Lambda_{\max}^2d_{\rm p}^2\big((\bU,\bV,\bW),(\bX,\bY,\bZ)\big)\log d
\end{eqnarray*}
which guarantees that 
\begin{equation}\label{ineq:bG-bC}
\|\bG-\bC\|_{\rm F}\leq C(\alpha \log d)^{1/2}\Lambda_{\max}d_{\rm p}\big((\bU,\bV,\bW),(\bX,\bY,\bZ)\big).%\leq 30\Lambda_{\max}(C\kappa_0)^{-1}.
\end{equation}
Recall that $\Lambda_{\max}\leq \bar{\Lambda}$ and $\Lambda_{\min}\geq \underline{\Lambda}$. We conclude that on event $\calE_1\cap\calE_2\cap \calE_3$,
\begin{eqnarray*}
\frac{1}{128}\|\bG-\bC\|_{\rm F}^2+\frac{1}{384}\underline{\Lambda}^2d_{\rm p}^2\big((\bU,\bV,\bW),(\bX,\bY,\bZ)\big)\\
\leq \frac{d_1d_2d_3}{n}F(\bX,\bY,\bZ)\leq C(\alpha\log d)\bar{\Lambda}^2d_{\rm p}^2\big((\bX,\bY,\bZ),(\bU,\bV,\bW)\big).
\end{eqnarray*}
\paragraph{Second statement.} 
Observe that
\begin{equation}\label{ineq:gradFlower}
\|{\rm grad}\ F(\bX,\bY,\bZ)\|_{\rm F}\geq \frac{\left\langle{\rm grad}\ F(\bX,\bY,\bZ), (\bD_\bX,\bD_\bY,\bD_\bZ)\right\rangle}{\left(\|\bD_\bX\|^2_{\rm F}+\|\bD_\bY\|^2_{\rm F}+\|\bD_\bZ\|^2_{\rm F}\right)^{1/2}}.
\end{equation}
%It therefore suffices to show that
%$$
%\left\langle{\rm grad}\ F(\bX,\bY,\bZ), (\bD_\bX,\bD_\bY,\bD_\bZ)\right\rangle\ge 
%$$
%
Write
$$
\bH=(\bD_\bX,\bY,\bZ)\cdot\bC+(\bX,\bD_\bY,\bZ)\cdot\bC+(\bX,\bY,\bD_\bZ)\cdot\bC.
$$
Then
$$
\left\langle{\rm grad}\ F(\bX,\bY,\bZ), (\bD_\bX,\bD_\bY,\bD_\bZ)\right\rangle=\left\langle \calP_\Omega(\hat{\bT}-\bT),\bH\right\rangle.
$$
Denote by
$$
\bH_1=(\bD_\bX,\bV,\bW)\cdot\bC+ (\bU,\bD_\bY,\bW)\cdot\bC+(\bU,\bV,\bD_\bZ)\cdot\bC
$$
and
\begin{align*}
\bH_2:=&(\bD_\bX,{\bf \Delta_Y},\bW)\cdot\bC+(\bD_\bX,\bV,{\bf \Delta_Z})\cdot \bC+ (\bD_\bX,{\bf \Delta_Y},{\bf \Delta_Z})\cdot \bC+({\bf \Delta_X},\bD_\bY,\bW)\cdot \bC\\
+& (\bU,\bD_\bY,{\bf \Delta_Z})\cdot \bC+({\bf \Delta_X},\bD_\bY,{\bf \Delta_Z})\cdot\bC+({\bf \Delta_X},\bV,\bD_\bZ)\cdot \bC+ (\bU,{\bf \Delta_Y},\bD_\bZ)\cdot\bC\\
+&({\bf \Delta_X},{\bf \Delta_Y},\bD_\bZ)\cdot \bC.
\end{align*}
Then, $\bH=\bH_1+\bH_2$ and $\bQ_\bT\bH_1=\bH_1$. We write
$$
\left\langle \calP_\Omega(\hat{\bT}-\bT),\bH\right\rangle=\left\langle \calP_\Omega\bQ_\bT(\hat{\bT}-\bT),\bH_1\right\rangle+\left\langle \calP_\Omega\bQ_\bT^\perp\hat{\bT},\bH_1\right\rangle+\left\langle \calP_\Omega(\hat{\bT}-\bT),\bH_2\right\rangle.
$$
Since $\bQ_\bT\bH_1=\bH_1$, we can show that under the event $\calE_1$,
$$
\left\langle \calP_\Omega\bQ_\bT(\hat{\bT}-\bT),\bH_1\right\rangle\geq \frac{d_1d_2d_3}{2n}\left\langle \bQ_\bT(\hat\bT-\bT),\bH_1\right\rangle.
$$
Based on the lower bound of $\left\langle \bQ_\bT(\hat\bT-\bT),\bH_1\right\rangle$ proved in Appendix~\ref{sec:QTH1}, we conclude that on event $\mathcal{E}_1\cap \calE_2\cap \calE_3$,
\begin{equation}\label{ineq:PQThatTTbH_1}
\left\langle \calP_\Omega\bQ_\bT(\hat{\bT}-\bT),\bH_1\right\rangle\geq \frac{n}{8d_1d_2d_3}\zeta_1\geq \frac{\Lambda_{\min}^2}{128}\frac{n}{d_1d_2d_3}d_{\rm p}^2\big((\bU,\bV,\bW),(\bX,\bY,\bZ)\big)
\end{equation}
where $\zeta_1:=\|({\bf \Delta}_\bX,\bV,\bW)\cdot \bC+(\bU,{\bf \Delta}_\bY,\bW)\cdot\bC+(\bU,\bV,{\bf \Delta}_\bZ)\cdot\bC\|_{\rm F}^2$ with (see Appendix~\ref{sec:QTH1})
\begin{equation}\label{ineq:zeta1lb}
\zeta_1\geq \frac{1}{16}\Lambda_{\min}^2d_{\rm p}^2\big((\bU,\bV,\bW),(\bX,\bY,\bZ)\big)
\end{equation}
on event $\calE_1\cap\calE_2\cap \calE_3$.
Moreover, by Cauchy-Schwarz inequality
$$
\left|\left\langle \calP_\Omega\bQ_\bT^\perp\hat{\bT},\bH_1\right\rangle\right|\leq \left\langle\calP_{\Omega}\bQ_\bT^\perp\hat\bT,\bQ_\bT^\perp\hat\bT \right\rangle^{1/2}\left\langle\calP_\Omega\bH_1,\bH_1 \right\rangle^{1/2}.
$$ 
%The upper bound of $\|\calP_\Omega\bQ_\bT^\perp\hat{\bT}\|_{\rm F}$ has been proved in (\ref{eq:bdL2}). It suffices to control $\|\calP_\Omega\bH_1\|_{\rm F}$ and
Observe that $\bQ_\bT\bH_1=\bH_1$. Therefore, under the event $\calE_1\cap \calE_2$,
$$
\langle \calP_\Omega\bQ_\bT\bH_1,\bQ_\bT\bH_1 \rangle^{1/2}\leq \sqrt{\frac{3 n}{2d_1d_2d_3}}\|\bH_1\|_{\rm F}.
$$
Recall the upper bound of $\|\bG-\bC\|_{\rm F}$ as in (\ref{ineq:bG-bC}) which implies that $\|\bG-\bC\|_{\rm F}\leq {\Lambda_{\min}}/{2}$ if 
$$
d_{\rm p}\big((\bU,\bV,\bW),(\bX,\bY,\bZ)\big)\leq (C\alpha\kappa_0\log d)^{-1}
$$
for a large enough $C>0$. As a result, on the event $\calE_1\cap \calE_2\cap \calE_3$,
\begin{equation}\label{ineq:LambdaC}
\frac{\Lambda_{\min}}{2}\leq \Lambda_{\min}(\bC)\leq \Lambda_{\max}(\bC)\leq 2\Lambda_{\max}
\end{equation}
 Then, on the event $\calE_1\cap \calE_2\cap \calE_3$,
\begin{align*}
\|\bH_1\|_{\rm F}\leq& \|({\bf \Delta}_\bX,\bV,\bW)\cdot \bC+(\bU,{\bf \Delta}_\bY,\bW)\cdot\bC+(\bU,\bV,{\bf \Delta}_\bZ)\cdot\bC\|_{\rm F}\\
&+\|({\bf \Delta}_\bX-\bD_\bX,\bV,\bW)\cdot \bC+(\bU,{\bf \Delta}_\bY-\bD_\bY,\bW)\cdot\bC+(\bU,\bV,{\bf \Delta}_\bZ-\bD_\bZ)\cdot\bC\|_{\rm F}\\
\leq&\sqrt{\zeta_1}+2\Lambda_{\max}\big(\|{\bf \Delta}_\bX-\bD_\bX\|_{\rm F}+\|{\bf \Delta}_\bY-\bD_\bY\|_{\rm F}+\|{\bf \Delta}_\bZ-\bD_\bZ\|_{\rm F}\big)\\
\leq& \sqrt{\zeta_1}+\sqrt{\zeta_1}8\kappa_0d_{\rm p}\big((\bU,\bV,\bW),(\bX,\bY,\bZ)\big)\leq 2\sqrt{\zeta_1}
\end{align*}
where we used the lower bound of $\zeta_1$ in (\ref{ineq:zeta1lb}). Moreover, it suffices to apply bound (\ref{ineq:PQTperphatTQTperphatT}) and (\ref{ineq:bG-bC}) to $\Big<\calP_{\Omega}\bQ_\bT^\perp\hat\bT, \bQ_\bT^\perp\hat\bT\Big>$. It is easy to check that as long as
$$d_{\rm p}\big((\bU,\bV,\bW),(\bX,\bY,\bZ)\big)\leq (C_1\alpha \kappa_0\log d)^{-1}$$ 
and
$$
n\geq C_1\Big(\alpha^3\kappa_0^2\mu_0^{3/2}r(r_1r_2r_3d_1d_2d_3)^{1/2}\log^{7/2}d+\alpha^6\kappa_0^4\mu_0^3r^2r_1r_2r_3d\log ^6d\Big)
$$
for a sufficiently large $C_1$,
\begin{equation}\label{ineq:QTperphatT}
\Big<\calP_{\Omega}\bQ_\bT^\perp\hat\bT, \bQ_\bT^\perp\hat\bT\Big>^{1/2}\leq \sqrt{\frac{n}{d_1d_2d_3}}\frac{\Lambda_{\min}}{128\sqrt{6}}d_{\rm p}\big((\bU,\bV,\bW),(\bX,\bY,\bZ)\big),
\end{equation}
under the event $\calE_1\cap \calE_2\cap \calE_3$. Due to the lower bound on $\zeta_1$ in (\ref{ineq:zeta1lb}),
\begin{eqnarray}\nonumber
\left|\left\langle \calP_\Omega\bQ_\bT^\perp\hat{\bT},\bH_1\right\rangle\right|&\leq&
\sqrt{6}\sqrt{\frac{n}{d_1d_2d_3}}\sqrt{\zeta_1}\sqrt{\frac{n}{d_1d_2d_3}}\frac{\Lambda_{\min}}{128\sqrt{6}}d_{\rm p}\big((\bU,\bV,\bW),(\bX,\bY,\bZ)\big)\\
&\leq& \frac{n}{32d_1d_2d_3}\zeta_1,\label{ineq:PQTperphatTbH1}
\end{eqnarray}
under the event $\calE_1\cap \calE_2\cap \calE_3$. It remains to control $\left|\left\langle \calP_\Omega(\hat\bT-\bT),\bH_2\right\rangle\right|$. The following fact (Cauchy-Schwarz inequality) on $\calE_2$ is obvious
\begin{equation}\label{ineq:PhatTTbH2upper}
\left|\left\langle \calP_\Omega(\hat\bT-\bT),\bH_2\right\rangle\right|\leq \left\langle\calP_\Omega(\hat\bT-\bT), \hat\bT-\bT\right\rangle^{1/2}\left\langle\calP_\Omega\bH_2,\bH_2 \right\rangle^{1/2}.
\end{equation}
On event $\calE_3$, by (\ref{ineq:PAA})
$$
\left\langle\calP_\Omega\bH_2, \bH_2\right\rangle\leq \frac{n}{d_1d_2d_3}\|\bH_2\|_{\rm F}^2+n\|\bH_2\|_{\rm F}^2\beta_n\Big(\frac{\|\bH_2\|_{\max}}{\|\bH_2\|_{\rm F}},\frac{\|\bH_2\|_{\star}}{\|\bH_2\|_{\rm F}}\Big).
$$
It is clear that
\begin{eqnarray*}
\|\bH_2\|_{\rm F}\leq 4\Lambda_{\max}\big(\|{\bf \Delta}_\bX\|_{\rm F}+\|{\bf \Delta}_\bY\|_{\rm F}+\|{\bf \Delta}_\bZ\|_{\rm F}\big)\big(\|{\bD}_\bX\|_{\rm F}+\|{\bD}_\bY\|_{\rm F}+\|{\bD}_\bZ\|_{\rm F}\big)\\
\leq 8\sqrt{2}\Lambda_{\max}d_{\rm p}^2\big((\bU,\bV,\bW),(\bX,\bY,\bZ)\big).
%\leq\sqrt{\zeta_1} 32\sqrt{2}\kappa_0d_{\rm p}\big((\bU,\bV,\bW),(\bX,\bY,\bZ)\big)\leq \sqrt{\zeta_1}.
\end{eqnarray*}
Meanwhile, by Appendix~\ref{sec:bH2}, 
$$
\|\bH_2\|_{\rm F}\leq 4\sqrt{6\zeta_1}d_{\rm p}\big((\bU,\bV,\bW),(\bX,\bY,\bZ)\big)+24\Lambda_{\max}d_{\rm p}^3\big((\bU,\bV,\bW),(\bX,\bY,\bZ)\big).
$$
Moreover, by Lemma~\ref{lemma:tensornorm},
$\|\bH_2\|_{\star}\leq 18r\|\bH_2\|_{\rm F}$. By \cite[Remark~8.1]{keshavan2009matrix}, 
$$
\max\{\mu(\bD_\bX),\mu(\bD_\bY),\mu(\bD_\bZ)\}\leq 55\mu_0.
$$
Thus, $\|\bH_2\|_{\max}\leq C_1\Lambda_{\max}\mu_0^{3/2}\sqrt{\frac{r_1r_2r_3}{d_1d_2d_3}}$ for an absolute constant $C_1>0$. Applying (\ref{ineq:PAA}), on the event $\calE_3$,
\begin{align*}
\left\langle\calP_\Omega\bH_2,\bH_2 \right\rangle\leq& \frac{n}{d_1d_2d_3}\|\bH_2\|_{\rm F}^2+C\alpha\|\bH_2\|_{\max}\|\bH_2\|_{\star}\Big(\sqrt{\frac{nd}{d_1d_2d_3}}\log d+\log^{3/2}d\Big)\\
\leq&C\cdot\bigg\{\frac{n}{d_1d_2d_3}\Lambda_{\max}^2 d_{\rm p}^6\big((\bU,\bV,\bW),(\bX,\bY,\bZ)\big)\\
&+\frac{n}{d_1d_2d_3}\zeta_1d_{\rm p}^2\big((\bU,\bV,\bW),(\bX,\bY,\bZ)\big)\\
&+\alpha r\mu_0^{3/2}\Lambda_{\max}^2\sqrt{\frac{r_1r_2r_3}{d_1d_2d_3}}\Big(\sqrt{\frac{nd}{d_1d_2d_3}}\log d+\log^{3/2}d\Big)d_{\rm p}^2\big((\bU,\bV,\bW),(\bX,\bY,\bZ)\big) \bigg\}.
\end{align*}
If
$$d_{\rm p}\big((\bU,\bV,\bW),(\bX,\bY,\bZ)\big)\leq (C_1\alpha\kappa_0\log d)^{-1}$$
and
$$
n\geq C_1\Big(\alpha^3\mu_0^{3/2}\kappa_0^4r(r_1r_2r_3d_1d_2d_3)^{1/2}\log^{7/2}d+\alpha^6\mu_0^3\kappa_0^8r^2r_1r_2r_3d\log^6d\Big),
$$
then the above bound can be simplified as
$$
\left\langle\calP_\Omega \bH_2,\bH_2 \right\rangle\leq \frac{n}{d_1d_2d_3}\Big(\frac{1}{5000^2C^2\alpha^2\log^2 d}\frac{\Lambda_{\min}^4}{\Lambda_{\max}^2}+C\zeta_1\Big)d_{\rm p}^2\big((\bU,\bV,\bW),(\bX,\bY,\bZ)\big).
$$
Moreover by (\ref{ineq:QTperphatT}), on the event $\calE_1\cap \calE_2\cap \calE_3$, 
\begin{align*}
\Big\langle\calP_\Omega (\hat\bT-\bT),\hat\bT-\bT\Big\rangle^{1/2}\leq&\|\calP_\Omega(\hat\bT-\bT)\|_{\rm F}\\
\leq& \|\calP_\Omega\bQ_\bT(\hat\bT-\bT)\|_{\rm F}+\|\calP_\Omega\bQ_\bT^\perp \hat\bT\|_{\rm F}\\
\leq& \sqrt{\frac{3C\alpha n\log d}{2d_1d_2d_3}}\|\bQ_\bT(\hat\bT-\bT)\|_{\rm F}\\
&+\sqrt{\frac{n}{d_1d_2d_3}}\frac{\Lambda_{\min}}{128\sqrt{6}}d_{\rm p}\big((\bU,\bV,\bW),(\bX,\bY,\bZ)\big)\\
\leq& 5\sqrt{\frac{n}{d_1d_2d_3}}\Lambda_{\max}(C\alpha\log d)d_{\rm p}\big((\bU,\bV,\bW),(\bX,\bY,\bZ)\big),
\end{align*}
where we used the following fact that, in the light of (\ref{eq:QThatTT}), (\ref{ineq:bG-bC}), (\ref{ineq:LambdaC}),
$$
\|\bQ_\bT(\hat\bT-\bT)\|_{\rm F}\leq \|\bG-\bC\|_{\rm F}+2\Lambda_{\max}d_{\rm p}\big((\bU,\bV,\bW),(\bX,\bY,\bZ)\big).
$$
Finally, on the event $\calE_1\cap \calE_2\cap \calE_3$, by (\ref{ineq:PhatTTbH2upper}),
\begin{eqnarray}
\left\langle\calP_\Omega(\hat\bT-\bT),\bH_2\right\rangle&\leq& \left\langle\calP_{\Omega}(\hat\bT-\bT),\hat\bT-\bT \right\rangle^{1/2}\left\langle\calP_\Omega\bH_2,\bH_2 \right\rangle^{1/2}\nonumber\\
&\leq& \frac{5}{5000}\frac{n}{d_1d_2d_3}\Big(\Lambda_{\min}^2+C\alpha\Lambda_{\max}\sqrt{\zeta_1}\log d\Big)d_{\rm p}^2\big((\bU,\bV,\bW),(\bX,\bY,\bZ)\big) \nonumber\\
&\leq& \frac{n}{32d_1d_2d_3}\zeta_1,\label{ineq:PhatTTbH2}
\end{eqnarray}
where we used bound (\ref{ineq:zeta1lb}) and the fact that
$$d_{\rm p}\big((\bU,\bV,\bW),(\bX,\bY,\bZ)\big)\leq (C\alpha\kappa_0\log d)^{-1}.$$
Putting (\ref{ineq:PQThatTTbH_1}), (\ref{ineq:PQTperphatTbH1}), (\ref{ineq:PhatTTbH2}) together, we conclude that on the event $\calE_1\cap \calE_2\cap \calE_3$,
\begin{eqnarray*}
\langle {\rm grad}\ F(\bX,\bY,\bZ), (\bD_\bX,\bD_\bY,\bD_\bZ)\rangle&=&\left \langle \calP_\Omega(\hat\bT-\bT),\bH\right\rangle\\
&\geq& \frac{n}{16d_1d_2d_3}\zeta_1\\
&\geq& \frac{n}{256d_1d_2d_3}\Lambda_{\min}^2d_{\rm p}^2\big((\bU,\bV,\bW),(\bX,\bY,\bZ)\big).
\end{eqnarray*}
Moreover, note that
$$\|\bD_\bX\|_{\rm F}+\|\bD_\bY\|_{\rm F}+\|\bD_\bZ\|_{\rm F}\leq 2d_{\rm p}\big((\bU,\bV,\bW),(\bX,\bY,\bZ)\big).$$ 
By (\ref{ineq:gradFlower}), we obtain
$$
\frac{d_1d_2d_3}{n}\|{\rm grad}\ F(\bX,\bY,\bZ)\|_{\rm F}\geq \frac{\Lambda_{\min}^2}{512}d_{\rm p}\big((\bU,\bV,\bW),(\bX,\bY,\bZ)\big),
$$
which concludes the proof since $\Lambda_{\min}\geq \underline{\Lambda}$.
\end{proof}
\medskip

\begin{proof}[Proof of Theorem \ref{th:GoGconv}]
We first note that the additional penalty function we imposed on $F$ does not change its local behavior in that Theorem \ref{thm:localF} still holds if we replace $F$ with $\tilde{F}$. In the light of Theorem \ref{thm:localF}, the first statement remains true for $\tilde{F}$ simply due to our choice of $\rho$. We now argue that the second statement also holds for $\tilde{F}$, more specifically,
$$
{d_1d_2d_3\over n}\left\|{\rm grad}\ \tilde{F}(\bX,\bY,\bZ)\right\|_{\rm F}\geq \frac{1}{512}\underline{\Lambda}^2 d_{\rm p}\Big((\bU,\bV,\bW), (\bX,\bY,\bZ)\Big),
$$
Observe that
$$
\|{\rm grad}\ \tilde{F}(\bX,\bY,\bZ)\|_{\rm F}\geq \frac{\Big<{\rm grad}\ F(\bX,\bY,\bZ), (\bD_\bX,\bD_\bY,\bD_\bZ)\Big>+\Big<{\rm grad}\ G(\bX,\bY,\bZ),(\bD_\bX,\bD_\bY,\bD_\bZ)\Big>}{\|\bD_\bX\|_{\rm F}+\|\bD_\bY\|_{\rm F}+\|\bD_\bZ\|_{\rm F}}.
$$
In proving Theorem \ref{thm:localF}, we showed that
$$
{d_1d_2d_3\over n}\frac{\Big<{\rm grad}\ F(\bX,\bY,\bZ), (\bD_\bX,\bD_\bY,\bD_\bZ)\Big>}{\|\bD_\bX\|_{\rm F}+\|\bD_\bY\|_{\rm F}+\|\bD_\bZ\|_{\rm F}}\geq \frac{1}{512}\underline{\Lambda}^2 d_{\rm p}\Big((\bU,\bV,\bW), (\bX,\bY,\bZ)\Big).
$$
It therefore suffices to show that
$$
\left\langle{\rm grad}\ G(\bX,\bY,\bZ),(\bD_\bX,\bD_\bY,\bD_\bZ)\right\rangle\ge 0.
$$
This follows the argument from \cite{keshavan2009matrix} and is omitted for brevity.

Now that Theorem \ref{thm:localF} holds for $\tilde{F}$, we know that $\tilde{F}(\bX,\bY,\bZ)$ has a unique stationary point in $\mathcal{N}(\delta,4\mu_0)$ at $(\bU,\bV,\bW)$ for $\delta\leq (C\alpha \kappa_0\log d)^{-1}$. Again, by a similar argument as that from \cite{keshavan2009matrix}, it can be show that all iterates $(\bX^{(k)},\bY^{(k)},\bZ^{(k)})\in \mathcal{N}(\delta/10,4\mu_0)$ and therefore Algorithm~\ref{alg:GoGalg} is just gradient descent with exact line search in $\mathcal{N}(\delta/10,4\mu_0)$. This suggests that Algorithm~\ref{alg:GoGalg} must converges to the unique stationary point $(\bU,\bV,\bW)$. See, e.g., \cite{luenberger2015linear}.
\end{proof}
\medskip

\bibliographystyle{plainnat}
\bibliography{refer}

\begin{thebibliography}{35}
\providecommand{\natexlab}[1]{#1}
\providecommand{\url}[1]{\texttt{#1}}
\expandafter\ifx\csname urlstyle\endcsname\relax
  \providecommand{\doi}[1]{doi: #1}\else
  \providecommand{\doi}{doi: \begingroup \urlstyle{rm}\Url}\fi

\bibitem[Absil et~al.(2008)Absil, Mahony, and Sepulchre]{absiletal08}
P.~Absil, R.~Mahony, and R.~Sepulchre.
\newblock \emph{Optimization Algorithms on Matrix Manifolds}.
\newblock Princeton University Press, 2008.

\bibitem[Anandkumar et~al.(2014)Anandkumar, Ge, Hsu, Kakade, and
  Telgarsky]{anandkumar2014tensor}
Animashree Anandkumar, Rong Ge, Daniel Hsu, Sham~M Kakade, and Matus Telgarsky.
\newblock Tensor decompositions for learning latent variable models.
\newblock \emph{Journal of Machine Learning Research}, 15\penalty0
  (1):\penalty0 2773--2832, 2014.

\bibitem[Barak and Moitra(2016)]{barak2016noisy}
Boaz Barak and Ankur Moitra.
\newblock Noisy tensor completion via the sum-of-squares hierarchy.
\newblock In \emph{29th Annual Conference on Learning Theory}, pages 417--445,
  2016.

\bibitem[Cand{\`e}s and Recht(2009)]{candes2009exact}
Emmanuel~J Cand{\`e}s and Benjamin Recht.
\newblock Exact matrix completion via convex optimization.
\newblock \emph{Foundations of Computational mathematics}, 9\penalty0
  (6):\penalty0 717--772, 2009.

\bibitem[Cand{\`e}s and Tao(2010)]{candes2010power}
Emmanuel~J Cand{\`e}s and Terence Tao.
\newblock The power of convex relaxation: Near-optimal matrix completion.
\newblock \emph{IEEE Transactions on Information Theory}, 56\penalty0
  (5):\penalty0 2053--2080, 2010.

\bibitem[Cohen and Collins(2012)]{CohenCollins12}
S.~Cohen and M.~Collins.
\newblock Tensor decomposition for fast parsing with latent-variable {PCFGS}.
\newblock In \emph{Advances in Neural Information Processing Systems}, 2012.

\bibitem[De~la Pena and Gin{\'e}(1999)]{de1999decoupling}
Victor De~la Pena and Evarist Gin{\'e}.
\newblock \emph{Decoupling: from dependence to independence}.
\newblock Springer Science \& Business Media, 1999.

\bibitem[de~la Pe{\~n}a and Montgomery-Smith(1995)]{de1995decoupling}
Victor~H de~la Pe{\~n}a and Stephen~J Montgomery-Smith.
\newblock Decoupling inequalities for the tail probabilities of multivariate
  {U}-statistics.
\newblock \emph{The Annals of Probability}, pages 806--816, 1995.

\bibitem[de~Silva and Lim(2008)]{silvalim08}
Vin de~Silva and Lek-Heng Lim.
\newblock Tensor rank and the ill-posedness of the best low-rank approximation
  problem.
\newblock \emph{SIAM Journal on Matrix Analysis and Applications}, 30\penalty0
  (3):\penalty0 1084--1127, 2008.

\bibitem[Edelman et~al.(1998)Edelman, Arias, and Smith]{edelman1998geometry}
Alan Edelman, Tom{\'a}s~A Arias, and Steven~T Smith.
\newblock The geometry of algorithms with orthogonality constraints.
\newblock \emph{SIAM journal on Matrix Analysis and Applications}, 20\penalty0
  (2):\penalty0 303--353, 1998.

\bibitem[Elden and Savas(2009)]{eldensavas09}
Lars Elden and Berkant Savas.
\newblock A {Newton-Grassmann} method for computing the best multilinear
  rank-($r_1,r_2,r_3$) approximation of a tensor.
\newblock \emph{SIAM Journal on Matrix Analysis and Applications}, 31\penalty0
  (2):\penalty0 248--271, 2009.

\bibitem[Gandy et~al.(2011)Gandy, Recht, and Yamada]{gandy2011tensor}
Silvia Gandy, Benjamin Recht, and Isao Yamada.
\newblock Tensor completion and low-n-rank tensor recovery via convex
  optimization.
\newblock \emph{Inverse Problems}, 27\penalty0 (2):\penalty0 025010, 2011.

\bibitem[Gin{\'e} and Nickl(2015)]{gine2015mathematical}
Evarist Gin{\'e} and Richard Nickl.
\newblock \emph{Mathematical foundations of infinite-dimensional statistical
  models}, volume~40.
\newblock Cambridge University Press, 2015.

\bibitem[Gross(2011)]{gross2011recovering}
David Gross.
\newblock Recovering low-rank matrices from few coefficients in any basis.
\newblock \emph{IEEE Transactions on Information Theory}, 57\penalty0
  (3):\penalty0 1548--1566, 2011.

\bibitem[Hillar and Lim(2013)]{hillarlim13}
C.~Hillar and Lek-Heng Lim.
\newblock Most tensor problems are {NP}-hard.
\newblock \emph{Journal of ACM}, 60\penalty0 (6):\penalty0 45, 2013.

\bibitem[Jain and Oh(2014)]{jain2014provable}
Prateek Jain and Sewoong Oh.
\newblock Provable tensor factorization with missing data.
\newblock In \emph{Advances in Neural Information Processing Systems}, pages
  1431--1439, 2014.

\bibitem[Keshavan et~al.(2009)Keshavan, Oh, and Montanari]{keshavan2009matrix}
Raghunandan~H Keshavan, Sewoong Oh, and Andrea Montanari.
\newblock Matrix completion from a few entries.
\newblock In \emph{2009 IEEE International Symposium on Information Theory},
  pages 324--328. IEEE, 2009.

\bibitem[Koltchinskii(2011)]{koltchinskii2011oracle}
Vladimir Koltchinskii.
\newblock \emph{Oracle Inequalities in Empirical Risk Minimization and Sparse
  Recovery Problems: {\'E}cole d'{\'E}t{\'e} de Probabilit{\'e}s de Saint-Flour
  XXXVIII-2008}.
\newblock Springer, 2011.

\bibitem[Kressner et~al.(2014)Kressner, Steinlechner, and
  Vandereycken]{kressner2014low}
Daniel Kressner, Michael Steinlechner, and Bart Vandereycken.
\newblock Low-rank tensor completion by {Riemannian} optimization.
\newblock \emph{BIT Numerical Mathematics}, 54\penalty0 (2):\penalty0 447--468,
  2014.

\bibitem[Li and Li(2010)]{LiLi10}
N.~Li and B.~Li.
\newblock Tensor completion for on-board compression of hyperspectral images.
\newblock In \emph{17th IEEE International Conference on Image Processing
  (ICIP)}, pages 517--520, 2010.

\bibitem[Liu et~al.(2013)Liu, Musialski, Wonka, and Ye]{liu2013tensor}
Ji~Liu, Przemyslaw Musialski, Peter Wonka, and Jieping Ye.
\newblock Tensor completion for estimating missing values in visual data.
\newblock \emph{IEEE Transactions on Pattern Analysis and Machine
  Intelligence}, 35\penalty0 (1):\penalty0 208--220, 2013.

\bibitem[Luenberger and Ye(2015)]{luenberger2015linear}
David~G Luenberger and Yinyu Ye.
\newblock \emph{Linear and nonlinear programming}, volume 228.
\newblock Springer, 2015.

\bibitem[Montanari and Sun(2016)]{montanari2016spectral}
Andrea Montanari and Nike Sun.
\newblock Spectral algorithms for tensor completion.
\newblock \emph{arXiv preprint arXiv:1612.07866}, 2016.

\bibitem[Mu et~al.(2014)Mu, Huang, Wright, and Goldfarb]{mu2014square}
Cun Mu, Bo~Huang, John Wright, and Donald Goldfarb.
\newblock Square deal: Lower bounds and improved convex relaxations for tensor
  recovery.
\newblock \emph{Journal of Machine Learning Research}, 1:\penalty0 1--48, 2014.

\bibitem[Rauhut and Stojanac(2015)]{rauhut2015tensor}
Holger Rauhut and {\v{Z}}eljka Stojanac.
\newblock Tensor theta norms and low rank recovery.
\newblock \emph{arXiv preprint arXiv:1505.05175}, 2015.

\bibitem[Rauhut et~al.(2016)Rauhut, Schneider, and Stojanac]{rauhut2016low}
Holger Rauhut, Reinhold Schneider, and Zeljka Stojanac.
\newblock Low rank tensor recovery via iterative hard thresholding.
\newblock \emph{arXiv preprint arXiv:1602.05217}, 2016.

\bibitem[Recht(2011)]{recht2011simpler}
Benjamin Recht.
\newblock A simpler approach to matrix completion.
\newblock \emph{Journal of Machine Learning Research}, 12\penalty0
  (Dec):\penalty0 3413--3430, 2011.

\bibitem[Savas and Lim(2010)]{savaslim10}
Berkant Savas and Lek-Heng Lim.
\newblock Quasi-newton methods on {Grassmannians} and multilinear
  approximations of tensors.
\newblock \emph{SIAM Journal on Matrix Analysis and Applications}, 32\penalty0
  (6):\penalty0 3352--3393, 2010.

\bibitem[Semerci et~al.(2014)Semerci, Hao, Kilmer, and Miller]{Semerci14}
O.~Semerci, N.~Hao, M.~Kilmer, and E.~Miller.
\newblock Tensor based formulation and nuclear norm regularizatin for
  multienergy computed tomography.
\newblock \emph{IEEE Transactions on Image Processing}, 23:\penalty0
  1678--1693, 2014.

\bibitem[Sidiropoulos and Nion(2010)]{SidNion10}
N.D. Sidiropoulos and N.~Nion.
\newblock Tensor algebra and multi-dimensional harmonic retrieval in signal
  processing for mimo radar.
\newblock \emph{IEEE Transactions on Signal Processing}, 58:\penalty0
  5693--5705, 2010.

\bibitem[Tomioka et~al.(2010)Tomioka, Hayashi, and
  Kashima]{tomioka2010estimation}
Ryota Tomioka, Kohei Hayashi, and Hisashi Kashima.
\newblock Estimation of low-rank tensors via convex optimization.
\newblock \emph{arXiv preprint arXiv:1010.0789}, 2010.

\bibitem[Tropp(2012)]{tropp2012user}
Joel~A Tropp.
\newblock User-friendly tail bounds for sums of random matrices.
\newblock \emph{Foundations of Computational Mathematics}, 12\penalty0
  (4):\penalty0 389--434, 2012.

\bibitem[Yu et~al.(2015)Yu, Wang, and Samworth]{yu2015useful}
Yi~Yu, Tengyao Wang, and Richard~J Samworth.
\newblock A useful variant of the {Davis--Kahan} theorem for statisticians.
\newblock \emph{Biometrika}, 102\penalty0 (2):\penalty0 315--323, 2015.

\bibitem[Yuan and Zhang(2016{\natexlab{a}})]{yuan2015tensor}
Ming Yuan and Cun-Hui Zhang.
\newblock On tensor completion via nuclear norm minimization.
\newblock \emph{Foundations of Computational Mathematics}, pages 1031--1068,
  2016{\natexlab{a}}.

\bibitem[Yuan and Zhang(2016{\natexlab{b}})]{yuan2016incoherent}
Ming Yuan and Cun-Hui Zhang.
\newblock Incoherent tensor norms and their applications in higher order tensor
  completion.
\newblock \emph{arXiv preprint arXiv:1606.03504}, 2016{\natexlab{b}}.

\end{thebibliography}
\clearpage

\appendix
\section{Proof of Lemma~\ref{lemma:tensornorm}}\label{sec:tensornorm}
The first claim is straightforward. It suffices to prove the second claim. Let $\bA=(\bU,\bV,\bW)\cdot\bC$ with $\bC\in\mathbb{R}^{r_1(\bA)\times r_2(\bA)\times r_3(\bA)}$ being the core tensor. Clearly, $\|\bA\|_{\star}=\|\bC\|_{\star}$ and $\|\bA\|_{\rm F}=\|\bC\|_{\rm F}$. Denote by $\bC_1,\ldots, \bC_{r_1(\bA)}\in\mathbb{R}^{r_2(\bA)\times r_3(\bA)}$ the mode-$1$ slices of $\bC$.  By convexity of nuclear norm,
$$
\|\bC\|_{\star}\leq \|\bC_1\|_{\star}+\ldots+\|\bC_{r_1(\bA)}\|_{\star}.
$$
As a result,
\begin{eqnarray*}
\|\bC\|_{\star}^2&\leq& r_1(\bA)\big(\|\bC_1\|_{\star}^2+\ldots+\|\bC_{r_1(\bA)}\|_{\star}^2\big)\\
&\leq& r_1(\bA)\big(r_2(\bA)\wedge r_3(\bA)\big)\big(\|\bC_1\|_{\rm F}^2+\ldots+\|\bC_{r_1(\bA)}\|_{\rm F}^2\big)\\
&=&r_1(\bA)\big(r_2(\bA)\wedge r_3(\bA)\big)\|\bC\|_{\rm F}^2.
\end{eqnarray*}
Therefore,
$$\|\bC\|_{\star}\leq \sqrt{r_1(\bA)\min\{r_2(\bA),r_3(\bA)\}}\|\bC\|_{\rm F}.$$
By the same process on mode-$2$ and mode-$3$ slices of $\bC$, we obtain 
$$
\|\bC\|_{\star}\leq \sqrt{r_2(\bA)\min\{r_1(\bA), r_3(\bA)\}}\|\bC\|_{\rm F},
$$
and
$$
\|\bC\|_{\star}\leq \sqrt{r_3(\bA)\min\{r_1(\bA), r_2(\bA)\}}\|\bC\|_{\rm F},
$$
which concludes the proof.

\section{Proof of Corollary~\ref{co:init}}
By Davis-Kahan Theorem \citep[see, e.g.,][]{yu2015useful},
$$
d_{\rm p}\big(\hat\bU,\bU\big)\leq \frac{2\sqrt{r_1}\|\hat\bN-\bM\bM^\top\|}{\sigma_{\min}(\bM\bM^\top)}.
$$
By choosing $m_1=d_1, m_2=d_2d_3$ in Theorem~\ref{th:spectral} and noticing that $n\geq C_1(\alpha+1)(d_1d_2d_3)^{1/2}$, then
$$
\|\hat\bN-\bM\bM^\top\|\leq C\alpha^2 {(d_1d_2d_3)^{3/2}\log d\over n}
\left[\left(1+\frac{d_1}{d_2d_3}\right)^{1/2}+\left(\frac{n}{d_2d_3\log d}\right)^{1/2}\right]\|\bM\|_{\max}^2
$$
with probability at least $1-d^{-\alpha}$. It suffices to control $\|\bM\|_{\max}$. Recall that $\mu(\bT)\leq \mu_0$, then
$$
\|\bM\|_{\max}=\|\bT\|_{\max}\leq \|\bT\|\mu_0^{3/2}\left(r_1r_2r_3\over d_1d_2d_3\right)^{1/2}.
$$
It is clear by definition that
$${\|\bT\|^2}/{\sigma_{\min}(\bM\bM^\top)}\leq \kappa^2(\bT)\leq \kappa_0^2.$$
As a result, the following bound holds with probability at least $1-d^{-\alpha}$,
\begin{align*}
d_{\rm p}\big(\hat\bU,\bU\big)\leq& 2C\alpha^2\mu_0^3\kappa_0^2r_1^{3/2}r_2r_3{(d_1d_2d_3)^{1/2}\log d\over n}\left[\left(1+\frac{d_1}{d_2d_3}\right)^{1/2}+\left(\frac{n}{d_2d_3\log d}\right)^{1/2}\right]\\
\leq& 2C\alpha^2\mu_0^3\kappa_0^2r_1^{3/2}r_2r_3\left[\frac{(d_1d_2d_3)^{1/2}\log d}{n}+\frac{d_1\log d}{n}+\left(\frac{d_1\log d}{n}\right)^{1/2}\right].
\end{align*}
The claim then follows.

\section{Proof of Lemma~\ref{emplem}}\label{sec:emplem}
For simplicity, define a random tensor $\bE\in\{0,1\}^{d_1\times d_2\times d_3}$ based on $\omega\in[d_1]\times[d_2]\times[d_3]$ such that $\bE(\omega)=1$ and all the other entries are $0$s. Let $\bE_1,\ldots,\bE_n$ be i.i.d. copies of $\bE$. Equivalently, we write
$$
\beta_n(\gamma_1,\gamma_2)=\underset{\bA\in\calK(\gamma_1,\gamma_2)}{\sup}\Big|\frac{1}{n}\sum_{i=1}^n\langle\bA,\bE_i\rangle^2-\mathbb{E}\langle\bA,\bE \rangle^2\Big|
$$
which is the upper bound of an empirical process indexed by $\mathcal{K}(\gamma_1,\gamma_2)$.
Define $\delta_{1,j}=2^j\delta_1^-$ for $j=0,1,2,\ldots,\floor{\log \frac{\delta_1^+}{\delta_1^-}}$ and $\delta_{2,k}=2^k\delta_2^-$ for $k=0,1,2,\ldots, \floor{\log \frac{\delta_2^+}{\delta_2^-}}$. For each $j, k$, we derive the upper bound of $\beta_n(\gamma_1,\gamma_2)$ with $\gamma_1\in[\delta_{1,j}, \delta_{1,j+1}]$ and $\gamma_2\in[\delta_{2,k},\delta_{2,k+1}]$. Following the union argument, we can make the bound uniformly true for $\gamma_1\in[\delta_1^-, \delta_1^+]$ and $\gamma_2\in[\delta_{2}^-, \delta_2^+]$.

Consider $\gamma_1\in[\delta_{1,j}, \delta_{1,j+1}]$, $\gamma_2\in[\delta_{2,k},\delta_{2,k+1}]$ and  observe that 
$$
\underset{\bA\in \mathcal{K}(\gamma_1,\gamma_2)}{\sup}\big|\langle\bA,\bE \rangle^2-\mathbb{E}\langle\bA,\bE \rangle^2\big|\leq \gamma_1^2.
$$ 
Moreover,
$$
\underset{{\bA}\in \mathcal{K}(\gamma_1,\gamma_2)}{\sup}{\rm Var}\big(\langle \bA,\bE\rangle^2\big)\leq \underset{{\bA}\in \mathcal{K}(\gamma_1,\gamma_2)}{\sup}\mathbb{E}\langle\bA,\bE \rangle^4\leq \frac{\gamma_1^2\|\bA\|_{\rm F}^2}{d_1d_2d_3}\leq \frac{\gamma_1^2}{d_1d_2d_3}.
$$
Applying Bousquet's version of Talagrand concentration inequality (see Theorem~3.3.9 in \cite{gine2015mathematical} and Theorem~2.6 in \cite{koltchinskii2011oracle}), with probability at least $1-e^{-t}$ for all $t\geq 0$,
$$
\beta_n(\gamma_1,\gamma_2)\leq 2\mathbb{E}\beta_n(\gamma_1,\gamma_2)+2\gamma_1\sqrt{\frac{t}{nd_1d_2d_3}}+2\gamma_1^2\frac{t}{n}.
$$
By the symmetrization inequality,
$$
\mathbb{E}\beta_n(\gamma_1,\gamma_2)\Big|\leq 2\mathbb{E}\underset{\bA\in \mathcal{K}(\gamma_1,\gamma_2)}{\sup}\Big|\frac{1}{n}\sum_{i=1}^n\varepsilon_{i}\langle \bA,\bE_i\rangle^2\Big|,
$$
where $\varepsilon_1,\ldots,\varepsilon_n$ are i.i.d Rademacher random variables.
Since $|\langle\bA,\bE \rangle|\leq \gamma_1$, by the contraction inequality,
$$
\mathbb{E}\beta_n(\gamma_1,\gamma_2)\leq 4\gamma_1\mathbb{E}\underset{\bA\in \mathcal{K}(\gamma_1,\gamma_2)}{\sup}\Big|\frac{1}{n}\sum_{i=1}^n\varepsilon_{i}\langle \bA,\bE_i\rangle\Big|.
$$
Denote ${\bf\Gamma}=n^{-1}\sum_{i=1}^n\varepsilon_i\bE_i\in\mathbb{R}^{d_1\times d_2\times d_3}$. Then, 
$$
\mathbb{E}\underset{\bA\in \mathcal{K}(\gamma_1,\gamma_2)}{\sup}\Big|\frac{1}{n}\sum_{i=1}^n\varepsilon_{i}\langle \bA,\bE_i\rangle\Big|\leq \mathbb{E}\underset{\bA\in \mathcal{K}(\gamma_1,\gamma_2)}{\sup}\|{\bf\Gamma}\|\|\bA\|_{\star}\leq \gamma_2\mathbb{E}\|{\bf\Gamma}\|.
$$
It is not difficult to show that, see e.g. \cite[Lemma~8]{yuan2015tensor} and \cite{yuan2016incoherent}
$$
\mathbb{E}\|{\bf \Gamma}\|\leq C\Big(\sqrt{\frac{d}{nd_1d_2d_3}}\log d+\frac{\log^{3/2}d}{n}\Big).
$$ 
The above bound holds as long as (see \cite{yuan2015tensor})
$$
n\geq C\Big\{\mu_0(r_1r_2r_3d_1d_2d_3)^{1/2}\log^{3/2}d+\mu_0^2r_1r_2r_3d\log^2d\Big\}.
$$
As a result, with probability at least $1-e^{-t}$,
$$
\beta_n(\gamma_1,\gamma_2)\leq C\gamma_1\gamma_2\Big(\sqrt{\frac{d}{nd_1d_2d_3}}\log d+\frac{\log^{3/2}d}{n}\Big)+2\gamma_1\sqrt{\frac{t}{nd_1d_2d_3}}+2\gamma_1^2\frac{t}{n}
$$
for $\gamma_1\in[\delta_{1,j}, \delta_{1,j+1}]$ and $\gamma_2\in[\delta_{2,k},\delta_{2,k+1}]$. Now, consider all the combinations of $j$ and $k$, we can make the upper bound uniformly for all $j$ and $k$ with adjusting $t$ to ${\bar t}$, and $C$ to $2C$.

\section{Proof of lower bound of $\langle \bQ_\bT(\hat\bT-\bT), \bH_1\rangle$}\label{sec:QTH1}
Recall that 
\begin{align*}
\langle \bQ_\bT(\hat\bT-\bT), \bH_1\rangle=& \Big<(\bU,\bV,\bW)\cdot (\bC-\bG)+({\bf \Delta}_\bX,\bV,\bW)\cdot \bC+(\bU,{\bf \Delta}_\bY,\bW)\cdot\bC\\
&+(\bU,\bV,{\bf \Delta}_\bZ)\cdot\bC
, (\bD_\bX,\bV,\bW)\cdot\bC+(\bU,\bD_\bY,\bW)\cdot\bC+(\bU,\bV,\bD_\bZ)\cdot\bC\Big>.
\end{align*}
Clearly, the right hand side can be written as $\zeta_1+\zeta_2+\zeta_3$ where
\begin{align*}
\zeta_1=&\|({\bf \Delta}_\bX,\bV,\bW)\cdot \bC+(\bU,{\bf \Delta}_\bY,\bW)\cdot\bC+(\bU,\bV,{\bf \Delta}_\bZ)\cdot\bC\|_{\rm F}^2\\
\zeta_2=&\Big<(\bU,\bV,\bW)\cdot (\bC-\bG), (\bD_\bX,\bV,\bW)\cdot\bC+(\bU,\bD_\bY,\bW)\cdot\bC+(\bU,\bV,\bD_\bZ)\cdot\bC\Big>\\
\zeta_3=&\Big<{\bf \Delta}_\bX,\bV,\bW)\cdot \bC+(\bU,{\bf \Delta}_\bY,\bW)\cdot\bC+(\bU,\bV,{\bf \Delta}_\bZ)\cdot\bC\\
,& (\bD_\bX-{\bf \Delta}_{\bX},\bV,\bW)\cdot\bC+(\bU,\bD_\bY-{\bf \Delta}_{\bY},\bW)\cdot\bC+(\bU,\bV,\bD_\bZ-{\bf \Delta}_{\bZ})\cdot\bC\Big>.
\end{align*}
Clearly,
\begin{align*}
\zeta_1\geq &\|({\bf \Delta}_\bX,\bV,\bW)\cdot \bC\|_{\rm F}^2+\|(\bU,{\bf \Delta}_\bY,\bW)\cdot\bC\|_{\rm F}^2+\|(\bU,\bV,{\bf \Delta}_\bZ)\cdot\bC\|_{\rm F}^2\\
&-2\Lambda_{\max}^2(\bC)\Big(\|\bU^\top {\bf \Delta}_\bX\|_{\rm F}\|\bV^\top{\bf \Delta}_\bY\|_{\rm F}+\|\bU^\top {\bf \Delta}_\bX\|_{\rm F}\|\bW^\top {\bf \Delta}_\bZ\|_{\rm F}+\|\bV^\top{\bf \Delta}_\bY\|_{\rm F}\|\bW^\top {\bf \Delta}_\bZ\|_{\rm F}\Big)\\
\geq& \Lambda_{\min}^2(\bC)\Big(\|{\bf \Delta}_\bX\|_{\rm F}^2+\|{\bf \Delta}_\bY\|_{\rm F}^2+\|{\bf \Delta}_\bZ\|_{\rm F}^2\Big)-8\Lambda_{\max}^2(\bC)d_{\rm p}^4\big((\bU,\bV,\bW),(\bX,\bY,\bZ)\big)
\end{align*}
where we used the fact that
$$\|\bU^\top{\bf \Delta}_\bX\|_{\rm F}\leq 2d_{\rm p}^2(\bU,\bX).$$
Recall from (\ref{ineq:LambdaC}) that on the event $\calE_1\cap \calE_2\cap \calE_3$, we have
$$\frac{\Lambda_{\min}}{2}\leq \Lambda_{\min}(\bC)\leq \Lambda_{\max}(\bC)\leq 2\Lambda_{\max}.$$
Then
$$
\zeta_1\geq \frac{1}{12}\Lambda_{\min}^2 d_{\rm p}^2\big((\bU,\bV,\bW),(\bX,\bY,\bZ)\big)-32\Lambda_{\max}^2 d_{\rm p}^4\big((\bU,\bV,\bW),(\bX,\bY,\bZ)\big).
%\geq \frac{1}{2}\sigma_{\min}^2(\bC)\frac{r_1r_2r_3}{r}\tilde\delta^2.
$$
It also implies that on the event $\calE_1\cap\calE_2\cap \calE_3$, 
\begin{equation}\label{ineq:zeta1lb1}
\zeta_1\geq \frac{1}{2}\Big(\|({\bf \Delta}_\bX,\bV,\bW)\cdot \bC\|_{\rm F}^2+\|(\bU,{\bf \Delta}_\bY,\bW)\cdot\bC\|_{\rm F}^2+\|(\bU,\bV,{\bf \Delta}_\bZ)\cdot\bC\|_{\rm F}^2\Big).
\end{equation}
We can control $|\zeta_3|$ in the same fashion. Indeed,
\begin{align*}
|\zeta_3|^2\leq& |\zeta_1| \Lambda_{\max}^2(\bC)(\|\bD_\bX-{\bf \Delta}_\bX\|_{\rm F}^2+\|\bD_\bY-{\bf \Delta}_\bY\|_{\rm F}^2+\|\bD_\bZ-{\bf \Delta}_\bZ\|_{\rm F}^2)\\
\leq& 4|\zeta_1|\Lambda_{\max}^2 d_{\rm p}^4\big((\bU,\bV,\bW),(\bX,\bY,\bZ)\big).
\end{align*}
If
$$d_{\rm p}\big((\bU,\bV,\bW),(\bX,\bY,\bZ)\big)\leq (C\alpha\kappa_0\log d)^{-1}$$
for large $C>0$, then under the event $\calE_1\cap \calE_2\cap \calE_3$,
$$
\zeta_1\geq \frac{1}{16}\Lambda_{\min}^2d_{\rm p}^2\big((\bU,\bV,\bW),(\bX,\bY,\bZ)\big) \quad  {\rm and}\quad |\zeta_3|\leq \frac{\zeta_1}{4}
$$
To control $\zeta_2$, recall that $\bX^\top \bD_\bX={\bf 0}, \bY^\top \bD_\bY={\bf 0}$ and $\bZ^\top \bD_\bZ={\bf 0}$. Then,
\begin{align*}
|\zeta_2|&\leq |\langle ({\bf \Delta}_\bX,\bV,\bW)\cdot (\bC-\bG), (\bD_\bX,\bV,\bW)\cdot\bC\rangle|\\
+& |\langle (\bU,{\bf \Delta}_\bY,\bW)\cdot (\bC-\bG), (\bU,\bD_\bY,\bW)\cdot\bC\rangle|+ |\langle (\bU,\bV,{\bf \Delta}_\bZ)\cdot (\bC-\bG), (\bU,\bV,\bD_\bZ)\cdot\bC\rangle|\\
&\leq 2\|\bC-\bG\|_{\rm F}\bigg\{\Big(\|({\bf \Delta}_\bX,\bV,\bW)\cdot\bC\|_{\rm F}+\|\bU,{\bf \Delta}_\bY,\bW)\cdot\bC\|_{\rm F}+\|(\bU,\bV,{\bf \Delta}_\bZ)\cdot\bC\|_{\rm F}\Big)\\
&+\Lambda_{\max}(\bC)\Big(\|\bD_\bX-{\bf \Delta}_\bX\|_{\rm F}+\|\bD_\bY-{\bf \Delta}_\bY\|_{\rm F}+\|\bD_\bZ-{\bf \Delta}_\bZ\|_{\rm F}\Big)\bigg\}d_{\rm p}\big((\bU,\bV,\bW),(\bX,\bY,\bZ)\big)\\
&\leq 2\|\bG-\bC\|_{\rm F}\sqrt{\zeta_1}d_{\rm p}\big((\bU,\bV,\bW),(\bX,\bY,\bZ)\big)+4\|\bC-\bG\|_{\rm F}\Lambda_{\max} d_{\rm p}^3\big((\bU,\bV,\bW),(\bX,\bY,\bZ)\big).
\end{align*}
Recall from (\ref{ineq:bG-bC}) that under the event $\calE_1\cap \calE_2\cap \calE_3$,
$$\|\bG-\bC\|_{\rm F}\leq C\Lambda_{\max}(\alpha\log d)^{1/2}d_{\rm p}\big((\bU,\bV,\bW),(\bX,\bY,\bZ)\big).$$
Therefore, $|\zeta_2|\leq \zeta_1/2$ in view of the lower bound of $\zeta_1$. In summary, under the event $\calE_1\cap \calE_2\cap\calE_3$,
$$
\langle\bQ_\bT(\hat\bT-\bT), \bH_1\rangle\geq \frac{1}{4}\zeta_1\geq \frac{1}{64}\Lambda_{\min}^2 d_{\rm p}^2\big((\bU,\bV,\bW),(\bX,\bY,\bZ)\big).
$$

\section{Upper bound of $\|\bH_2\|_{\rm F}$}\label{sec:bH2}
It is shown in (\ref{ineq:zeta1lb1}) that if $d_{\rm p}\big((\bU,\bV,\bW),(\bX,\bY,\bZ)\big)\leq (C\alpha\kappa_0\log d)^{-1}$, then
$$
\zeta_1\geq \frac{1}{2}\Big(\|({\bf \Delta}_\bX,\bV,\bW)\cdot \bC\|_{\rm F}^2+\|(\bU,{\bf \Delta}_\bY,\bW)\cdot \bC\|_{\rm F}^2+\|(\bU,\bV,{\bf \Delta}_\bZ)\cdot \bC\|_{\rm F}^2\Big).
$$
Observe that
$$
\|({\bf \Delta}_\bX,\bV,\bW)\cdot\bC\|_{\rm F}^2=\|\calM_2(\bC)({\bf \Delta}_\bX\otimes \bW)\|_{\rm F}=\|\calM_3(\bC)({\bf \Delta}_\bX\otimes \bV)\|_{\rm F}
$$
which implies that
$$
\zeta_1\geq \frac{1}{6}\Big(\|\calM_2(\bC)({\bf \Delta}_\bX\otimes \bW)\|_{\rm F}+\|\calM_3(\bC)(\bU\otimes {\bf \Delta}_\bY)\|_{\rm F}+\|\calM_1(\bC)( \bV\otimes {\bf \Delta}_\bZ)\|_{\rm F}\Big)^2
$$
By definition of $\bH_2$, we obtain
\begin{align*}
\|\bH_2\|_{\rm F}\leq& \|\calM_1(\bC)({\bf \Delta}_\bY\otimes \bW)\|_{\rm F}\|\bD_\bX\|_{\rm F}+\|\calM_1(\bC)(\bV\otimes {\bf \Delta}_\bZ)\|_{\rm F}\|\bD_\bX\|_{\rm F}\\
+& \|\calM_2(\bC)({\bf \Delta}_\bX\otimes \bW)\|_{\rm F}\|\bD_\bY\|_{\rm F}+\|\calM_2(\bC)(\bU\otimes {\bf \Delta}_\bZ)\|_{\rm F}\|\bD_\bY\|_{\rm F}\\
+&\|\calM_3(\bC)({\bf \Delta}_\bX\otimes \bV)\|_{\rm F}\|\bD_\bZ\|_{\rm F}+\|\calM_3(\bC)(\bU\otimes {\bf \Delta}_\bY)\|_{\rm F}\|\bD_\bZ\|_{\rm F}\\
+&24\Lambda_{\max}d_{\rm p}^3\big((\bU,\bV,\bW),(\bX,\bY,\bZ)\big)
\end{align*}
where we used the fact $\Lambda_{\max}(\bC)\leq 2\Lambda_{\max}$ from (\ref{ineq:LambdaC}). Clearly,
\begin{align*}
\|\bH_2\|_{\rm F}\leq& 2\sqrt{6\zeta_1}\big(\|\bD_\bX\|_{\rm F}+\|\bD_\bY\|_{\rm F}+\|\bD_\bZ\|_{\rm F}\big)+24\Lambda_{\max}d_{\rm p}^3\big((\bU,\bV,\bW),(\bX,\bY,\bZ)\big)\\
\leq& 4\sqrt{6\zeta_1}d_{\rm p}\big((\bU,\bV,\bW),(\bX,\bY,\bZ)\big)+24\Lambda_{\max}d_{\rm p}^3\big((\bU,\bV,\bW),(\bX,\bY,\bZ)\big).
\end{align*}
\end{document}